\documentclass{article}

\makeatletter

\def\Headings#1#2{\def\ps@mypagestyle{\let\@mkboth\@gobbletwo%
\def\@oddhead{\hfill {\small\scshape #1} \hfill}%
\def\@oddfoot{\hfill \small\rmfamily \thepage \hfill}%
\def\@evenhead{\hfill {\small\scshape #2} \hfill}%
\def\@evenfoot{\hfill \small\rmfamily \thepage \hfill}}%
\pagestyle{mypagestyle}}

\renewcommand\footnoterule{\kern-3\p@ \hrule \@width \textwidth \kern 2\p@}

\makeatother

\usepackage[T1]{fontenc}
\usepackage[stretch=10]{microtype}
\usepackage{amsmath,amsthm,amssymb}
\usepackage{enumerate}
\usepackage{bbm}
\usepackage{geometry}
\usepackage{titlesec}
\usepackage{cite}
\usepackage[colorlinks,linkcolor=blue]{hyperref}

\titleformat{\section}%
{\large\bfseries}{\thesection.}{1ex}{}

\titleformat{\subsection}%
{\normalsize\bfseries}{\thesubsection}{2ex}{}

\renewenvironment{abstract}%
{{\centering\bfseries Abstract\par}\vspace{1ex}%
	\bgroup\leftskip 40pt\rightskip 40pt\small\noindent\ignorespaces}%
{\par\egroup\vskip 1ex}

\newtheorem{theorem}{Theorem}[section]
\newtheorem{corollary}[theorem]{Corollary}
\newtheorem{lemma}[theorem]{Lemma}
\newtheorem{proposition}[theorem]{Proposition}
\theoremstyle{definition}
\newtheorem{assumption}[theorem]{Assumption}
\newtheorem{definition}[theorem]{Definition}

\renewcommand{\epsilon}{\varepsilon}

\newcommand{\B}{\mathbb{B}}
\newcommand{\N}{\mathbb{N}}
\newcommand{\R}{\mathbb{R}}

\newcommand{\IndFct}[1]{\mathbbm{1}_{#1}}
\newcommand{\norm}[1]{\left\|#1\right\|}

\newcommand{\weight}{\omega}
\newcommand{\WNorm}[1]{\left\|#1\right\|_{\weight}}
\newcommand{\sequ}{\mathcal{C}}
\newcommand{\sequW}{\sequ_{\weight}}
\newcommand{\Orb}{\mathcal{O}}
\newcommand{\U}[1]{\mathcal{U}_{#1}}
\newcommand{\Op}[2]{T_{#1}^{#2}}
\newcommand{\FPmap}{\Phi}
\newcommand{\CSMmap}{\Psi}

\newcommand{\timeZero}{\alpha}
\newcommand{\timeOne}{\beta}
\newcommand{\target}{\mathfrak{f}}

\newcommand{\NNskel}{\theta}
\newcommand{\NNskelAlt}{\vartheta}
\newcommand{\NNfct}{\mathcal{N}_{\NNskel}}

\newcommand{\Loss}{\mathcal{L}}
\newcommand{\Grad}{\mathcal{G}}
\newcommand{\J}{\mathcal{J}}

\newcommand{\stepsize}{\gamma}
\newcommand{\SGDstep}[1]{\Theta_{#1}^{\stepsize,\NNskel}}
\newcommand{\SGDstepAlt}[1]{\Theta_{#1}^{\stepsize,\NNskelAlt}}
\newcommand{\SetSaddle}{\mathcal{S}}

\newcommand{\ConvTraj}{\mathcal{T}}

\newcommand{\leaky}{\zeta}
\newcommand{\NNleakyfct}{\mathcal{N}_{\NNskel}^{\leaky}}
\newcommand{\leakyLoss}{\mathcal{L}^{\leaky}}
\newcommand{\leakyGrad}{\mathcal{G}^{\leaky}}

\title{\Large{\bfseries{Gradient descent provably escapes saddle points \\ in the training of shallow ReLU networks}}}

\author{Patrick Cheridito\footnote{Department of Mathematics and RiskLab, ETH Zurich, Switzerland} \qquad Arnulf Jentzen\footnote{School of Data Science and Shenzhen Research Institute of Big Data, The Chinese University of Hong Kong, Shenzhen (CUHK-Shenzhen), China} \footnote{Applied Mathematics: Institute for Analysis and Numerics, University of M{\"u}nster, Germany} \qquad Florian Rossmannek$^*$\footnote{School of Physical and Mathematical Sciences, Nanyang Technological University, Singapore}
}

\Headings{Gradient descent provably escapes saddle points in the training of shallow ReLU networks}{Cheridito, Jentzen, Rossmannek}

\date{}

\begin{document}

\maketitle

\begin{abstract}
	Dynamical systems theory has recently been applied in optimization to prove that gradient descent algorithms bypass so-called strict saddle points of the loss function.
However, in many modern machine learning applications, the required regularity conditions are not satisfied.
In this paper, we prove a variant of the relevant dynamical systems result, a center-stable manifold theorem, in which we relax some of the regularity requirements.
We explore its relevance for various machine learning tasks, with a particular focus on shallow rectified linear unit (ReLU) and leaky ReLU networks with scalar input.
Building on a detailed examination of critical points of the square integral loss function for shallow ReLU and leaky ReLU networks relative to an affine target function, we show that gradient descent circumvents most saddle points.
Furthermore, we prove convergence to global minima under favourable initialization conditions, quantified by an explicit threshold on the limiting loss.
\end{abstract}




\section{Introduction}

In many machine learning frameworks, which constitute an important class of nonconvex optimization problems, gradient descent and its variants are the go-to algorithms for the training process.
However, due to the nonconvex nature of these problems, there are no a priori universal guarantees for convergence of these algorithms.
Both local minima and saddle points of the loss function used for training can prevent the algorithms from reaching a global minimum.
Originally, local minima were assumed to pose the greater challenge, but recent results suggest that saddle points are the main obstacle; \cite{ChoHenMatBenLeCun2015,DauPasGulChoGanBen2014,VenBanBru2019}.

An important ingredient in tackling this problem is {\itshape strictness} of saddle points, meaning that the Hessian of the loss function has a strictly negative eigenvalue at these saddle points.
The strictness ensures that there is a direction along which the loss surface declines significantly.
We explain this in more detail further below.
Under the strictness assumption, a stochastic version of gradient descent with suitable noise in each step has the ability to avoid saddle points because the noise ensures that we discover the declining direction;
\cite{Pemantle1990,GeHuangJinYuan2015,JinGeNetKakJor2017}.
The noise even guarantees a polynomial speed in escaping these saddle points; \cite{DuJinLeeJorSinPoc2017}.

In the case of vanilla gradient descent, there is no noise to rely on, and one needs more involved analytic methods.
A useful tool in this context is the {\itshape stable manifold theorem}, which is a cornerstone of classical dynamical systems theory; \cite{Shub1987}.
It has recently been applied to prove that vanilla gradient descent with suitable random initialization avoids strict saddle points with probability one if the loss function is sufficiently regular;
\cite{LeeSimJorRecht2016,PanPil2017}.
We remark that the applicability of the stable manifold theorem goes beyond vanilla gradient descent; see
\cite{DaskPan2018,LeePanPilSimJorRecht2019,ONeillWright2019}
for its application to variants of gradient descent and other first-order methods.

Accumulation points of gradient descent trajectories are critical points.
Under typical assumptions like boundedness of trajectories and, e.g., validity of {\L}ojasiewicz-type inequalities, it is also known that trajectories converge to a critical point; see
\cite{DavDruKakLee2020,FraGarPey2015}
for the stochastic and
\cite{AbsMahAnd2005,LeiHuLiTang2019}
for the nonstochastic version.
It follows that, with probability one, these limit critical points are local minima or nonstrict saddle points.
The strictness assumption has been discussed in the literature and has been shown to hold in a variety of settings;
e.g, in matrix recovery
\cite{BhoNeySre2016,GeJinZheng2017,SunQuWright2017a},
phase retrieval
\cite{SunQuWright2018},
tensor decomposition
\cite{GeHuangJinYuan2015},
shallow quadratic networks,
\cite{DuLee2018,SolJavLee2019},
and deep linear networks
\cite{BahRauTerWest2021,Kawaguchi2016}.
In particular, nonstrict saddle points appear to be less common than strict ones, and the above results shrink the gap to proving convergence to local minima.

Whereas strictness has been discussed in abundance, less attention has been given to the regularity assumptions imposed on the loss function.
In
\cite{DaskPan2018,LeeSimJorRecht2016,LeePanPilSimJorRecht2019,ONeillWright2019},
the loss function is taken to be twice continuously differentiable with a globally Lipschitz continuous gradient, and in \cite{PanPil2017} these conditions are assumed to hold on a forward-invariant convex open set.
This level of regularity makes the classical dynamical systems theory directly applicable to gradient descent algorithms.
However, in many modern machine learning applications, the loss function is neither twice continuously differentiable nor is its gradient uniformly Lipschitz continuous on suitable invariant sets.

The intuition behind this theory becomes clearer if one pictures the linearization of the gradient descent map $f(x) = x - \stepsize \nabla\Loss(x)$ (the function describing one step of the algorithm with stepsize $\stepsize$ and loss function $\Loss$) around a saddle point $z$.
Note that $z$, being a critical point of the loss function, is a fixed point of the gradient descent map.
For simplicity of the presentation, assume $z=0$.
The first-order Taylor approximation of $f$ around the origin reads $f(x) \approx f'(0)x = (I - \stepsize \nabla^2\Loss(0)) x$, where $I$ denotes the identity matrix.
Therefore, after neglecting second and higher-order terms, the behavior of the next step $f(x)$ can be determined by looking at $x$ in the eigenspace decomposition of the matrix $f'(0)$.
If the saddle point $0$ is strict, then $\nabla^2\Loss(0)$ has a strictly negative eigenvalue, so $f'(0)$ has an eigenvalue strictly greater than 1.
Thus, there is a direction in the linearization along which we move away from the origin.
This means that the only way to actually move towards the origin is if one moves {\itshape inside} a so-called center-stable manifold.
Loosely speaking, a center-stable manifold is a manifold whose tangent space at the origin is the span $E^{cs}$ of the eigenvectors of $f'(0)$ for eigenvalues of absolute value less than or equal to 1.
The span $E^{cs}$ is the center-stable space of the linearization, and a\footnote{While the linear subspace $E^{cs}$ is unique, a center-stable manifold is not; \cite{Shub1987}.} center-stable manifold takes into account second and higher-order terms.
The final step of the approach consists in showing that the set of initializations, from which the gradient descent trajectory eventually enters this center-stable manifold, has measure zero.

A restrictive assumption implicitly used in the above argument is that $f'(0)$ is non-degenerate.
Indeed, if $f'(0)$ is degenerate, then, for $x$ in the kernel of $f'(0)$, nothing can be said about $f(x)$ without considering second-order terms.
In
\cite{DaskPan2018,LeeSimJorRecht2016,LeePanPilSimJorRecht2019,ONeillWright2019,PanPil2017},
this non-degeneracy assumption is guaranteed to hold by requiring $\nabla\Loss$ to be globally Lipschitz continuous.
Then, $I - \stepsize \nabla^2\Loss(0)$ cannot be degenerate for sufficiently small $\stepsize$ compared to the Lipschitz constant.\footnote{We remark that local Lipschitz continuity at the origin is sufficient to guarantee that $I-\stepsize\nabla^2\Loss(0)$ is non-degenerate for small $\stepsize$.
But we want to study many saddle points in an unbounded set simultaneously, and $\stepsize$ would depend on the local Lipschitz constant around each of those saddle points.
Therefore, to guarantee $\stepsize \ne 0$, we would need a uniform upper bound on these local Lipschitz constants, which essentially amounts to a global bound.}
But global Lipschitz continuity of $\nabla\Loss$ is a strong assumption and does not hold in many machine learning frameworks.
In conclusion, one of the main difficulties on the side of the dynamical systems theory is to provide a variant of the center-stable manifold theorem that works even if $f'(0)$ is degenerate.
To this end, we extend a result of \cite{PanPilWang2019}, which we present in Theorem \ref{intro_center_manifold_thrm}.
Therein, observe that $f'(x)$ need only be non-degenerate almost everywhere but not necessarily at the saddle points $x \in \SetSaddle$ of interest.

\begin{theorem}
\label{intro_center_manifold_thrm}
	Let $d \in \N$,
let $\norm{\cdot} \colon \R^d \rightarrow \R$ be the standard norm on $\R^d$,
let $f \colon \R^d \rightarrow \R^d$ be a function,
let $(X^x_n)_{(n,x) \in \N_0 \times \R^d} \subseteq \R^d$ be given by $X^x_0 = x$ and $X^x_{n+1} = f(X^x_n)$,
let $V \subseteq U \subseteq \R^d$ be open sets,
assume that $\R^d \backslash V$ has Lebesgue measure zero,
assume $f|_U \in C^1(U,\R^d)$,
assume that $U \ni x \mapsto f'(x) \in \R^{d \times d}$ is  locally Lipschitz continuous,
assume for all $x \in V$ that $\det(f'(x)) \ne 0$,
let $\SetSaddle \subseteq \{x \in U \colon f(x)=x\}$,
and assume for all $x \in \SetSaddle$ that the matrix $f'(x)$ is symmetric and has an eigenvalue whose absolute value is strictly greater than 1.
Then, the set $\{ x \in \R^d \colon ( \exists\, y \in \SetSaddle \colon \limsup_{n \rightarrow \infty} \norm{X^x_n-y} = 0 ) \}$ has Lebesgue measure zero.
\end{theorem}

This theorem entails a range of possible applications, including phase retrieval, matrix recovery, and learning with neural networks.
Shallow networks with a rectified linear unit (ReLU) activation will be studied more closely as they lend themselves to developing further techniques.
Namely, in the sketch further above, we implicitly use that $f'(0)$ is diagonalizable, which is guaranteed if $I - \stepsize \nabla^2\Loss(0)$ is a real symmetric matrix.
But this requires $\Loss$ to be twice differentiable at the origin, which is too strong a property to ask for.
To tackle this problem, we have to modify the gradient descent map and consider $f(x) = x - \stepsize \Grad(x)$, where $\Grad$ is a modification of $\nabla\Loss$.
The function $\Grad$ may not arise as the gradient of any scalar-valued function.
Therefore, we need to ensure explicitly that $\Grad'$ is symmetric at the origin so that $f'(0)$ is still diagonalizable.
Of course, we have to ensure that, upon replacing $\nabla\Loss$ by its modification $\Grad$, we do not loose information about the dynamics of the original gradient descent algorithm.
To obtain the necessary strictness of (in some sense) most saddle points, we rely on a classification of saddle points from \cite{CheJenRos2022JNLS}.
To apply this classification, we need to restrict our attention to shallow ReLU networks on the $L^2$-loss with respect to a one-dimensional affine target function.
Combining all of the above, we prove in Theorem \ref{GD_avoids_saddle_points} that the gradient descent algorithm almost surely avoids most saddle points in this framework, where almost surely is understood with respect to a random initialization that is absolutely continuous with respect to the Lebesgue measure.
In the analogous framework in which ReLU networks are replaced with leaky ReLU networks, the same result is deduced in Theorem \ref{leaky_GD_avoids_saddle_points}.
However, the proof simplifies for leaky ReLU as there is no need to work with a modified gradient.

Building more intricately on the classification of critical points from \cite{CheJenRos2022JNLS}, we proceed to deduce convergence of the algorithm to a global minimum under a suitable initialization as stated in Theorem \ref{intro_convergence_thrm} below.
Let us explain the notation used in that theorem.
A shallow network with $N$ hidden neurons and scalar input and output is a collection of weights and biases, represented by a vector $\NNskel \in \R^{3N+1}$.
The realization of such a network is the function $\NNfct$.
The map $\Loss$ is the squared $L^2$-loss measured against a target function $\target$.
As $\Loss$ is not differentiable everywhere, we take $\Grad$ to be the left gradient of $\Loss$, that is we take partial directional derivatives from the left.
This specific choice is for the sake of the presentation, but in the main body of this article $\Grad$ may take coordinate-wise any values when $\Loss$ is not differentiable.
Finally, $\SGDstep{k+1} = \SGDstep{k} - \stepsize \Grad(\SGDstep{k})$ is the gradient descent algorithm with stepsize $\stepsize$ and initial value $\NNskel$.

\begin{theorem}
\label{intro_convergence_thrm}
	Let $N \in \N$, $\timeZero,\timeOne \in \R$ satisfy $\timeZero < \timeOne$ and $N/2 \in \N$,
for every $\NNskel =(\NNskel_1,\dots,\NNskel_{3N+1}) \in \R^{3N+1}$ let $\NNfct \in C([\timeZero,\timeOne],\R)$ be given by $\NNfct(x) = \NNskel_{3N+1} + \sum_{j=1}^{N} \NNskel_{2N+j} \max\{\NNskel_jx+\NNskel_{N+j},0\}$,
let $\target \in C([\timeZero,\timeOne],\R)$ be affine,
let $\Loss \in C(\R^{3N+1},\R)$ be given by $\Loss(\NNskel) = \int_{\timeZero}^{\timeOne} ( \NNfct(x) - \target(x) )^2 \, dx$,
let $\Grad \colon \R^{3N+1} \rightarrow \R^{3N+1}$ be the left gradient of $\Loss$,
and let $(\SGDstep{k})_{(k,\stepsize,\NNskel) \in \N_0 \times (0,\infty) \times \R^{3N+1}} \subseteq \R^{3N+1}$ be given by $\SGDstep{0} = \NNskel$ and $\SGDstep{k+1} = \SGDstep{k} - \stepsize \Grad(\SGDstep{k})$.
Then, for Lebesgue almost all $\stepsize \in (0,\infty)$ and Lebesgue almost all
\begin{equation}
	\NNskel \in \left\{ \NNskelAlt \in \R^{3N+1} \colon (\SGDstepAlt{k})_{k \in \N_0} \text{ is convergent and } \lim_{k \rightarrow \infty} \Loss(\SGDstepAlt{k}) < \textstyle{\frac{[\target'(\timeZero)]^2 (\timeOne-\timeZero)^3}{12(N-1)^4}} \right\}
\end{equation}%
it holds that $\lim_{k \rightarrow \infty} \Loss(\SGDstep{k}) = 0$.
\end{theorem}

We remark that the conclusion of Theorem \ref{intro_convergence_thrm} is void if the target function $\target$ is constant.
In this case, every critical point of $\Loss$ is a global minimum and there is nothing to prove; \cite{CheJenRieRos2022}.


\subsection{Structure}

The remainder of this article is structured as follows.
In Section \ref{section_center_manifold}, we state our variant of the center-stable manifold theorem and deduce Theorem \ref{intro_center_manifold_thrm}.
We explore applications of Theorem \ref{intro_center_manifold_thrm} in Section \ref{section_applications}.
Section \ref{section_network_training} examines the shallow ReLU network framework as a particular example and concludes with proving Theorem \ref{intro_convergence_thrm}.
The results for ReLU networks are extended to leaky ReLU networks in Section \ref{leaky_section_network_training}.
Finally, Section \ref{section_proof_center_manifold_thrm} contains the proof of the center-stable manifold theorem.

\subsection{Notation}

We denote by $\norm{\cdot}$ the Euclidean norm when applied to vectors and the operator norm induced by the Euclidean norm when applied to matrices.
Throughout this article, we fix a dimension $d \in \N$ and write $I \in \R^{d \times d}$ for the identity matrix.
The closed ball around a point $x \in \R^d$ with radius $r \in (0,\infty)$ is denoted $\B_r(x) = \{y \in \R^d \colon \norm{y-x} \leq r \}$.
A discrete dynamical system is written as follows.
For every function $f \colon \R^d \rightarrow \R^d$, we denote by $f^k \colon \R^d \rightarrow \R^d$, $k \in \N_0$, the functions that satisfy for all $k \in \N_0$ that $f^0 = \mathrm{id}_{\R^d}$ and $f^{k+1} = f \circ f^k$.
To describe critical points of a function $\Loss \colon \R^d \rightarrow \R$, we use the following terminology.
Local extrema refer to nonstrict local extrema; a point $x \in \R^d$ is called a critical point of $\Loss$ if $\Loss$ is differentiable at $x$ with $\nabla\Loss(x) = 0$; and a critical point is called a saddle point if it is not a local extremum.


\section{A center-stable manifold theorem}
\label{section_center_manifold}

The core of this section is a variant of the stable manifold theorem.
The novelty is that we do not require the dynamical system to be a local diffeomorphism as is the case in the classical formulation \cite{Shub1987}.
Specifically, the Jacobian may be degenerate at the fixed point under consideration.
This comes at the expense of less regularity of the center-stable manifold.
Indeed, the graph in Theorem \ref{center_manifold_thrm} is only proved to be Lipschitz-regular.
Our variant is an extension of the corresponding statement in \cite{PanPilWang2019}.
The exact regularity requirement needed is a certain local Lipschitz condition on the remainder term of the first-order Taylor expansion of the dynamical system around a fixed point:

\begin{assumption}
\label{center_manifold_asmp}
	Let $f \colon \R^d \rightarrow \R^d$ be a function and let $\SetSaddle \subseteq \{ x \in \R^d \colon f(x) = x\}$.
Assume for all $z \in \SetSaddle$ that $f$ is differentiable at $z$, that the matrix $f'(z) \in \R^{d \times d}$ is diagonalizable over $\R$ and has an eigenvalue of absolute value strictly greater than 1, and that for all $\epsilon \in (0,\infty)$ there exists $r_{\epsilon} \in (0,\infty)$ so that the map $\B_{r_{\epsilon}}(z) \rightarrow \R^d$, $x \mapsto f(x) - z - f'(z) (x-z)$ is $\epsilon$-Lipschitz continuous.
\end{assumption}

For all $z \in \SetSaddle$, denote by $E^{cs}_z \subseteq \R^d$ the span of those eigenvectors of $f'(z)$ associated with eigenvalues that lie in $[-1,1]$ (the center-stable space) and by $E^u_z \subseteq \R^d$ the span of those eigenvectors of $f'(z)$ associated with eigenvalues that lie in $\R \backslash [-1,1]$ (the unstable space).
Then, $\R^d = E^{cs}_z \oplus E^u_z$.
Under Assumption \ref{center_manifold_asmp}, we have $0 \leq \dim(E^{cs}_z) \leq d-1$ for all $z \in \SetSaddle$.
Now, we can state our version of the center-stable manifold theorem.

\begin{theorem}[Center-stable Lipschitz manifold]
\label{center_manifold_thrm}
	Let Assumption \ref{center_manifold_asmp} hold and let $z \in \SetSaddle$.
Then, there exists an $r \in (0,\infty)$ and a Lipschitz continuous map $\CSMmap \colon E^{cs}_z \rightarrow E^u_z$ such that
\begin{equation*}
	\{x \in \R^d \colon f^k(x) \in \B_r(z) \text{ for all } k \in \N_0\} \subseteq \mathrm{Graph}(\CSMmap).
\end{equation*}%
\end{theorem}

This theorem states that all those points, whose orbits under the dynamical system remain close to $z$, lie in the graph of a Lipschitz function, whose domain is a linear space of dimension between 0 and $d-1$.
We defer the proof to Section \ref{section_proof_center_manifold_thrm}.
In Theorem \ref{center_manifold_thrm}, we considered a single point $z \in \SetSaddle$.
We obtain a statement about all points in $\SetSaddle$ simultaneously the same way it was done in
\cite{PanPil2017,LeePanPilSimJorRecht2019}, using second-countability of Euclidean space.
For completeness, we repeat the argument to prove Corollary \ref{basin_union_measure_zero}.

\begin{corollary}
\label{basin_union_measure_zero}
	Let Assumption \ref{center_manifold_asmp} hold.
Then, there exists a set $W \subseteq \R^d$ of Lebesgue measure zero such that
\begin{equation*}
	\Big\{ x \in \R^d \colon \lim_{k \rightarrow \infty} f^k(x) \in \SetSaddle \Big\} \subseteq \bigcup_{k \in \N_0} f^{-k}(W).
\end{equation*}%
\end{corollary}

\begin{proof}
	By Theorem \ref{center_manifold_thrm}, for all $z \in \SetSaddle$, there exists an open neighborhood $U_z \subseteq \R^d$ of $z$ and a Lipschitz continuous map $\CSMmap_z \colon E^{cs}_z \rightarrow E^u_z$ such that $\{x \in \R^d \colon f^k(x) \in U_z \text{ for all } k \in \N_0\} \subseteq \mathrm{Graph}(\CSMmap_z)$.
Now, $\bigcup_{z \in \SetSaddle} U_z$ is an open cover of $\SetSaddle$ and, by second-countability of $\R^d$, there exists a countable subcover $\bigcup_{n \in \N_0} U_{z_n}$.
Set $W = \bigcup_{n \in \N_0} \mathrm{Graph}(\CSMmap_{z_n})$.
If $y \in \{ x \in \R^d \colon \lim_{k \rightarrow \infty} f^k(x) \in \SetSaddle \}$, then there exist $k,n \in \N_0$ such that for all $m \in \N_0$ we have $f^m(f^k((y)) = f^{m+k}(y) \in U_{z_n}$.
Thus, $f^k(y) \in \mathrm{Graph}(\CSMmap_{z_n})$ and, hence, $y \in f^{-k}(W)$.
Lastly, the set $W$, being a countable union of graphs, has Lebesgue measure zero.
\end{proof}

Note that Corollary \ref{basin_union_measure_zero} is a statement about the stable set of $\SetSaddle$ and not its center-stable set.
However, the proof of the corollary relies on the center-stable manifolds from Theorem \ref{center_manifold_thrm} and would not work with the stable manifolds.

The goal is to show, under reasonable assumptions, that the set $\{ x \in \R^d \colon \lim_{k \rightarrow \infty} f^k(x) \in \SetSaddle \}$ has Lebesgue measure zero.
This follows from the previous corollary if we can ensure that preimages of measure zero sets under $f^k$ have themselves measure zero.
This is certainly true for local diffeomorphisms
(see \cite{PanPil2017,LeePanPilSimJorRecht2019}), but we do not want to exclude the possibility that the dynamical system has a degenerate Jacobian at points in $\SetSaddle$.
Fortunately, it is sufficient to have a non-degenerate Jacobian almost everywhere (but potentially at no point in $\SetSaddle$), as the next lemma shows.

\begin{lemma}
\label{preserving_measure_zero_sets}
	Let $f \colon \R^d \rightarrow \R^d$ be a function and suppose there exists an open set $V \subseteq \R^d$, whose complement has Lebesgue measure zero, such that $f$ is continuously differentiable on $V$ with $\det(f'(x)) \ne 0$ for all $x \in V$.
Then, for any set $W \subseteq \R^d$ of Lebesgue measure zero, the set $f^{-1}(W)$ also has Lebesgue measure zero.
\end{lemma}

\begin{proof}
	First, note that $f$ is Lebesgue measurable because it is continuous on a subset of full measure.
It suffices to show that the set $f^{-1}(W) \cap V$ has Lebesgue measure zero.
By the assumptions, the restriction of $f$ to $V$ is a local $C^1$-diffeomorphism.
This and second-countability of $\R^d$ guarantee the existence of a countable open cover $\bigcup_{n \in \N} D_n = V$ of $V$ such that, for all $n \in \N$, the restriction $f|_{D_n} \colon D_n \rightarrow f(D_n)$ is a $C^1$-diffeomorphism.
By the integral transformation theorem, we have, for all $n \in \N$,
\begin{equation*}
	\int_{f^{-1}(W) \cap D_n} |\!\det(f'(x))| \, dx = \int_{f(f^{-1}(W) \cap D_n)} dx \leq \int_W dx = 0.
\end{equation*}%
Since $\det(f'(x)) \ne 0$ for all $x \in V$, this implies that $f^{-1}(W) \cap D_n$ has Lebesgue measure zero for all $n \in \N$ and, hence, so does $f^{-1}(W) \cap V$.
\end{proof}

With Corollary \ref{basin_union_measure_zero}, we conclude that if Assumption \ref{center_manifold_asmp} holds as well as the assumption of Lemma \ref{preserving_measure_zero_sets}, then the stable set $\{ x \in \R^d \colon \lim_{k \rightarrow \infty} f^k(x) \in \SetSaddle \}$ of $\SetSaddle$ has Lebesgue measure zero.
We finish this section by applying this result to a class of dynamical systems that includes the gradient descent algorithm, which will be of interest in the next section.

\begin{proposition}
\label{GD_basin_measure_zero}
	Let $f \colon \R^d \rightarrow \R^d$ be a function and suppose there exist open sets $V \subseteq U \subseteq \R^d$, whose complements have Lebesgue measure zero, such that $f$ is continuously differentiable on $U$ with a locally Lipschitz continuous Jacobian, which is non-degenerate on $V$.
Let $\SetSaddle \subseteq \{x \in U \colon f(x)=x\}$
and assume for all $x \in \SetSaddle$ that $f'(x)$ is symmetric and has an eigenvalue of absolute value strictly greater than 1.
Then, the set $\{ x \in \R^d \colon \lim_{k \rightarrow \infty} f^k(x) \in \SetSaddle\}$ has Lebesgue measure zero.
\end{proposition}

\begin{proof}
	The result is immediate from Corollary
\ref{basin_union_measure_zero} and Lemma \ref{preserving_measure_zero_sets} once we have verified that Assumption \ref{center_manifold_asmp} is fulfilled.
Let $z \in \SetSaddle$ and let $R(x) = f(x) - z - f'(z)(x-z)$ be the remainder term of the first-order Taylor expansion of $f$ around $z$.
We need to verify that for a given $\epsilon \in (0,\infty)$ we can take $r_{\epsilon} \in (0,\infty)$ so small that the restriction of $R$ to $\B_{r_{\epsilon}}(z)$ is $\epsilon$-Lipschitz continuous.
Take $r \in (0,\infty)$ so that $\B_r(z) \subseteq U$.
For all $x \in \B_r(z)$, note that
\begin{equation*}
	R(x) = \int_0^1 \big[ f'(z+s(x-z)) - f'(z) \big] (x-z) \, ds.
\end{equation*}%
Denote the Lipschitz constant of $f'$ on $\B_r(z)$ by $L \in [0,\infty)$
Then, for all $x,y \in \B_r(z)$,
\begin{equation*}
\begin{split}
	\norm{R(x) - R(y)}
	&= \norm{ \int_0^1 \big[ f'(z+s(x-z)) - f'(z) \big] (x-y) + \big[ f'(z+s(x-z)) - f'(z+s(y-z)) \big] (y-z) \, ds }
	\\
	&\leq \int_0^1 Ls \norm{x-z} \norm{x-y} + Ls \norm{x-y} \norm{y-z} \, ds
	\leq Lr \norm{x-y}.
\end{split}
\end{equation*}%
So, given $\epsilon \in (0,\infty)$, we pick $r_{\epsilon} = \min\{r,\epsilon L^{-1}\}$.
\end{proof}


\section{Applications}
\label{section_applications}

The prime applications we have in mind are gradient descent algorithms in machine learning frameworks.
Many of those frameworks can be posed as optimizing a loss function $\Loss \colon \R^d \rightarrow [0,\infty)$ that is a piecewise polynomial.
Indeed, we often encounter an empirical risk $\Loss \in C(\R^d,\R)$ on data points $(X_i,Y_i)_{i=1,\dots,M}$ given by
\begin{equation*}
	\Loss(x) = c \sum_{i=1}^M \norm{ p(x,X_i) - Y_i }^2
\end{equation*}
for a constant $c > 0$ and a piecewise polynomial $p$.
In particular, there is a set $U \subseteq \R^d$ of full measure on which $\Loss$ is smooth.
The same holds for the gradient descent map given by $f_{\stepsize}(x) = x - \stepsize \nabla\Loss(x)$ on $U$.
Furthermore, the map $P \colon U \times \R \rightarrow \R$ given by $P(x,\stepsize) = \det(I-\stepsize\nabla^2\Loss(x))$ is a piecewise polynomial, which is not constantly zero on any connected component of $U \times \R$.
Hence, its zero set has Lebesgue measure zero, which implies that for almost every $\stepsize \in (0,\infty)$ there exists a set $V \subseteq U$ of full measure on which $f_{\stepsize}'(x)$ is non-degenerate.
Now, we can apply Proposition \ref{GD_basin_measure_zero} to deduce that gradient descent almost surely avoids all strict saddle points of $\Loss$ in $U$.
Even if $p$ is smooth everywhere, the previous theoretical results from the literature surveyed in the introduction had not been applicable because the gradient of $\Loss$\ is not globally Lipschitz continuous and the Jacobian of the gradient descent map cannot be guaranteed to be non-degenerate.

Often times, a regularizer is added to the empirical risk.
But popular regularizers are themselves piecewise polynomials, so the above still applies.
Furthermore, the above argument also works if we consider an empirical $L^1$-loss instead of an $L^2$-loss due to almost everywhere differentiability of locally Lipschitz continuous functions on level sets; \cite{EvansGabriely2015}.
We survey some concrete examples below.

\subsection{Phase retrieval}

In phase retrieval, one seeks to recover a vector $x$ satisfying $(X_i^T x)^2 = Y_i$ for given vectors $X_i$ and scalars $Y_i$.
This amounts to taking the function $p(x,X_i) = (X_i^T x)^2$ in the empirical risk above.
Strictness of saddle points has been studied in \cite{SunQuWright2018} for the $L^2$-loss.
If we consider the $L^1$-loss instead, \cite{DavisDruPaq2020} considered a subgradient method to tackle this optimization problem.
Since the $L^1$-loss is smooth almost everywhere, the subgradient method in \cite{DavisDruPaq2020} agrees with a gradient descent algorithm almost everywhere, albeit with a varying stepsize in each iteration.
Proposition \ref{GD_basin_measure_zero} extends to certain non-autonomous dynamical systems that cover the case of gradient descent with a varying stepsize.

\subsection{Matrix recovery}

In matrix recovery, the function $p$ in the empirical risk takes the form $p(x,X_i) = x^T X_i$, where the input $x$ and the data point $X_i$ are $d_1 \times d_2$-dimensional matrices.
A problem often studied is low-rank matrix recovery, in which $p$ is restricted to the set of matrices of a given rank $r < \min(d_1,d_2)$.
This problem becomes more tractable upon replacing the domain $\{x \in \R^{d_1 \times d_2} \colon \mathrm{rank}(x) = r\}$ by $\R^{d_1 \times r} \times \R^{d_2 \times r}$ and setting $p(U,V,X_i) = V U^T X_i$.
There are other possibilities to parametrize the rank constraint and for which the curvature at saddle points has been studied;
\cite{BhoNeySre2016,GeJinZheng2017,SunQuWright2017a}.
As for phase retrieval, Proposition \ref{GD_basin_measure_zero} extends to cover subgradient methods used on the $L^1$-loss as done, for example, in \cite{LiZhuSoVidal2020}.

\subsection{Networks with piecewise polynomial activation}

Consider functions parametrized by networks $\R^d \rightarrow C(\R^m,\R^o)$, $\NNskel \mapsto \NNfct$ of a given architecture.
If the activation function in the networks is a piecewise polynomial, then $p(\NNskel,X_i) = \NNfct(X_i)$ is a piecewise polynomial.
Thus, networks embed in the above framework for any architecture and any input and output dimensions.
While empirical evidence suggests that many saddle points are strict (see \cite{DauPasGulChoGanBen2014}), this is proved only in some special cases; for example, for shallow quadratic networks \cite{DuLee2018,SolJavLee2019} and for deep linear networks \cite{BahRauTerWest2021,Kawaguchi2016}.
A more complete theoretical understanding of saddle points eludes us; \cite{ChoLeCunBen2015}.

In all of the above examples, the empirical risk was the objective to be optimized, which lend itself nicely for the framework of Proposition \ref{GD_basin_measure_zero}.
As the number of data points grows larger, the empirical risk converges to the true risk given by an integral of the error of the model against the true target.
In the next section, we study the true loss for a specialized model class, namely networks with one hidden layer and ReLU activation against a one-dimensional affine target function.
The true loss is no longer a piecewise polynomial but will be shown to be a rational function almost everywhere, which permits an argument similar as for the empirical risk.
More importantly, in this special framework, we have a much more explicit knowledge of the saddle points.
Note that for the empirical risk we deduced that gradient descent avoids saddle points of $\Loss$ in $U$.
However, there may be saddle points of interest outside of $U$.
In the next section, we showcase additional methods how to extend the result to saddle points outside of $U$.


\section{Gradient descent for shallow ReLU networks}
\label{section_network_training}

We now turn to studying shallow ReLU networks.
Throughout this section, suppose $d=3N+1$ for an $N \in \N$ and fix $\timeZero,\timeOne \in \R$ with $\timeZero < \timeOne$.
Then, $\R^d$ represents the space of all shallow networks with $N$ hidden neurons.
We will always write a network $\NNskel \in \R^{3N+1}$ as $\NNskel = (w,b,v,c)$, where $w,b,v \in \R^N$ and $c \in \R$.
The realization of a network $\NNskel$ is the function $\NNfct \in C(\R,\R)$ given by
\begin{equation}
\label{network_realization}
	\NNfct(x) = c + \sum_{j=1}^{N} v_j \max\{w_jx+b_j,0\}.
\end{equation}%
Fix $\target \in C([\timeZero,\timeOne],\R)$.
We denote by $\Loss \in C(\R^d,\R)$ the squared $L^2$-loss with target function $\target$, that is
\begin{equation}
\label{loss_function}
	\Loss(\NNskel) = \int_{\timeZero}^{\timeOne} ( \NNfct(x) - \target(x) )^2 \, dx.
\end{equation}%
To discuss regularity properties of the loss function, it is convenient to recall the following definition, which has been introduced in \cite{CheJenRos2022JNLS}.
Motivation and discussion of these notions can be found there.

\newpage

\begin{definition}
\label{types_of_neurons}
	Let $\NNskel = (w,b,v,c) \in \R^{3N+1}$ and $j \in \{1,\dots,N\}$.
Then, we denote by $I_j$ the set given by $I_j = \{x \in [\timeZero,\timeOne] \colon w_jx+b_j \geq 0 \}$, we say that the $j^{th}$ hidden neuron of $\NNskel$ is\newline
\begin{minipage}{0.49\linewidth}
\vspace{0.2cm}
\begin{itemize}\itemsep = -0.2em
\item {\itshape inactive} if $I_j = \emptyset$,
\item {\itshape semi-inactive} if $\# I_j = 1$,
\item {\itshape semi-active} if $w_j = 0 < b_j$,
\item {\itshape active} if $w_j \ne 0 < b_j + \max\{w_j\timeZero,w_j\timeOne\}$,
\item {\itshape type-1-active} if $w_j \ne 0 \leq b_j + \min\{w_j\timeZero,w_j\timeOne\}$,
\end{itemize}
\vspace{0cm}
\end{minipage}
\hfill
\begin{minipage}{0.49\linewidth}
\vspace{0.2cm}
\begin{itemize}\itemsep = -0.2em
\item {\itshape type-2-active} if $\emptyset \ne I_j \cap (\timeZero,\timeOne) \ne (\timeZero,\timeOne)$,
\item {\itshape degenerate} if $|w_j| + |b_j| = 0$,
\item {\itshape non-degenerate} if $|w_j| + |b_j| > 0$,
\item {\itshape flat} if $v_j = 0$,
\item {\itshape non-flat} if $v_j \ne 0$,
\end{itemize}
\vspace{0cm}
\end{minipage}
and we say that $t \in \R$ is the {\itshape breakpoint} of the $j^{th}$ hidden neuron of $\NNskel$ if $w_j \ne 0 = w_jt+b_j$.
\end{definition}

It was shown in \cite{CheJenRos2022JNLS} that $\Loss$ is differentiable at all coordinates corresponding to non-degenerate or flat degenerate neurons.
In general, the loss fails to be differentiable at non-flat degenerate neurons.
To apply the dynamical systems theory, we need a function defined on the whole $\R^d$.
Thus, we need to work with a generalized gradient of $\Loss$.
There are many different choices for such a generalized gradient.
Here, we actually do not specify a choice, but only require that our generalized gradient agrees with partial derivatives of $\Loss$ coordinate-wise.
So, throughout this section, let $\Grad \colon \R^d \rightarrow \R^d$ satisfy for all $\NNskel \in \R^d$ and all $j \in \{1,\dots,N\}$ such that the $j^{th}$ neuron of $\NNskel$ is non-degenerate or flat degenerate that
\begin{equation*}
	\Grad_j(\NNskel) = \frac{\partial}{\partial w_j} \Loss(\NNskel), \quad \Grad_{N+j}(\NNskel) = \frac{\partial}{\partial b_j} \Loss(\NNskel), \quad \Grad_{2N+j}(\NNskel) = \frac{\partial}{\partial v_j} \Loss(\NNskel), \quad \Grad_{3N+1}(\NNskel) = \frac{\partial}{\partial c} \Loss(\NNskel).
\end{equation*}%
The map $\Grad$ may take any values at coordinates of non-flat degenerate neurons.
The dynamical system we are interested in is the gradient descent step $f_{\stepsize}(\NNskel) = \NNskel - \stepsize \Grad(\NNskel)$ for some given stepsize $\stepsize \in (0,\infty)$.

One crucial aspect of Theorem \ref{center_manifold_thrm} is that we do not need the dynamical system to be a local diffeomorphism, let alone differentiable everywhere.
However, the dynamical system ought to be differentiable at the saddle points of $\Loss$ we are interested in.
Where $\Grad = \nabla\Loss$, differentiability of $f_{\stepsize}$ means two times differentiability of $\Loss$.
Even though $\Loss$ is twice differentiable on a set of full measure (see \cite{JenRie2022JMLR}), some of the saddle points lie outside of that full-measure set.
More precisely, it is semi-inactive neurons that cause the regularity problems.
The resulting nonexistence of the Hessian of $\Loss$ urges us to work with suitably modified dynamical systems.
The idea is to replace entries of $\Grad$ that correspond to semi-inactive neurons of a given saddle point but to keep the remaining entries as they are.
For technical reasons, a prescribed set $J \subseteq \{1,\dots,N\}$ of semi-inactive neurons is split into two subsets $J_+$ and $J_-$, each containing those semi-inactive neurons with $w_j>0$ and $w_j<0$, respectively.
The exact formula for the modified gradient $\Grad^J$ is given below.
The new dynamical system $f_{\stepsize,J}(\NNskel) = \NNskel - \stepsize \Grad^J(\NNskel)$ no longer coincides with the original gradient descent $f_{\stepsize}$, but we will be able to recover information about the dynamics of $f_{\stepsize}$ from $f_{\stepsize,J}$; see Lemma \ref{GD_regularity_increase} below.
Now, let $\J$ be the set
\begin{equation*}
	\J = \{ (J_+,J_-) \colon J_+,J_- \subseteq \{1,\dots,N\} \text{ such that } J_+ \cap J_- = \emptyset \}.
\end{equation*}%
In this general definition, $J_+$ and $J_-$ are any disjoint subsets of $\{1,\dots,N\}$ and only later they will carry the interpretation of representing semi-inactive neurons.
For any $J = (J_+,J_-) \in \J$, let $\Grad^J \colon \R^d \rightarrow \R^d$ satisfy for all $\NNskel \in \R^d$ and $j \in \{1,\dots,N\}$ that $\Grad^J_{3N+1}(\NNskel) = \Grad_{3N+1}(\NNskel)$ and
\begin{equation*}
	\left( \Grad^J_j , \Grad^J_{N+j} , \Grad^J_{2N+j} \right)(\NNskel) =
	\begin{cases}
		\left( \Grad_j , \Grad_{N+j} , \Grad_{2N+j} \right)(\NNskel) &\text{if } j \notin J_+ \cup J_- \text{ or } w_j = 0, \\
		2 \int_{t_j}^{\timeOne} ( v_jx , v_j , w_jx+b_j )(\NNfct(x)-\target(x)) \, dx  &\text{if } j \in J_+ \text{ and } w_j \ne 0, \\
		2 \int_{\timeZero}^{t_j} ( v_jx , v_j , w_jx+b_j )(\NNfct(x)-\target(x)) \, dx  &\text{if } j \in J_- \text{ and } w_j \ne 0.
	\end{cases}
\end{equation*}%
Note that $\Grad^{(\emptyset,\emptyset)} = \Grad$.
If a neuron $j \in J_+ \cup J_-$ is semi-inactive or type-2-active with the sign of $w_j$ matching the sign in the subscript of $J_{\pm}$, then $\Grad^J$ agrees with $\Grad$ in the coordinates of the $j^{th}$ neuron.
Thus, we did not actually change $\Grad$ at semi-inactive neurons with matching signs, but we changed $\Grad$ at inactive neurons in a way that $\Grad^J$ becomes differentiable at semi-inactive neurons (which are neighbored by inactive neurons).
In the next lemma, we leverage that the original dynamical system $f_{\stepsize}$ does not alter coordinates of inactive neurons to show how to infer dynamical information about $f_{\stepsize}$ from its modifications.

\begin{lemma}
\label{GD_regularity_increase}
	Let $\SetSaddle \subseteq \R^d$ and, for all $J = (J_+,J_-) \in \J$, let $\SetSaddle_J \subseteq \SetSaddle$ contain all networks $\NNskel \in \SetSaddle$ such that $J_+$ is exactly the set of neurons of $\NNskel$ that are semi-inactive with $w_j > 0$ and $J_-$ is exactly the set of neurons of $\NNskel$ that are semi-inactive with $w_j < 0$.
Then, $\SetSaddle = \bigcup_{J \in \J} \SetSaddle_J$ and
\begin{equation*}
	\left\{ \NNskel \in \R^d \colon \lim_{k \rightarrow \infty} f_{\stepsize}^k(\NNskel) \in \SetSaddle \right\} \subseteq \bigcup_{J \in \J} \bigcup_{n \in \N_0} f_{\stepsize}^{-n}\left( \left\{ \NNskel \in \R^d \colon \lim_{k \rightarrow \infty} f_{\stepsize,J}^k(\NNskel) \in \SetSaddle_J \right\} \right).
\end{equation*}%
\end{lemma}

\begin{proof}
	That $\SetSaddle = \bigcup_{J \in \J} \SetSaddle_J$ is clear.
Suppose $\NNskel_0 \in \{ \NNskel \in \R^d \colon \lim_{k \rightarrow \infty} f_{\stepsize}^k(\NNskel) \in \SetSaddle \}$ and let $\NNskel_{\infty} \in \SetSaddle$ be the limit point of $f_{\stepsize}^k(\NNskel_0)$ as $k \rightarrow \infty$.
Take $J \in \J$ with $\NNskel_{\infty} \in \SetSaddle_J$ and abbreviate $\NNskel_k = f_{\stepsize}^k(\NNskel_0)$.
Note that $f_{\stepsize}$ does not change coordinates of inactive neurons.
More precisely, for all $j \in \{1,\dots,N\}$ and $k,n \in \N_0$, if the $j^{th}$ neuron of $\NNskel_n$ is inactive, then $\NNskel_n$ and $\NNskel_{n+k}$ agree in the $(w_j,b_j,v_j)$-coordinates.
Furthermore, any sufficiently small neighborhood of a semi-inactive neuron contains only inactive, semi-inactive, and type-2-active neurons.
It follows from this that there exists an $n \in \N_0$ such that for all $k \in \N_0$ and all $j \in J_+ \cup J_-$ the $j^{th}$ neuron of $\NNskel_{n+k}$ is type-2-active or semi-inactive with $\mathrm{sgn}(w_j)$ matching the subscript of $J_{\pm}$ in both cases.
Then, $\Grad^J(\NNskel_{n+k}) = \Grad(\NNskel_{n+k})$ for all $k \in \N_0$.
In particular, $\NNskel_{n+k} = f_{\stepsize,J}^k(\NNskel_n)$ for all $k \in \N_0$ and, hence, $\NNskel_n \in \{ \NNskel \in \R^d \colon \lim_{k \rightarrow \infty} f_{\stepsize,J}^k(\NNskel) \in \SetSaddle_J \}$.
\end{proof}

Subsequently, we need to accomplish two objectives.
First, we need to verify that the dynamical systems theory (Proposition \ref{GD_basin_measure_zero}) is applicable to each $f_{\stepsize,J}$ to deduce that the sets $\{ \NNskel \in \R^d \colon \lim_{k \rightarrow \infty} f_{\stepsize,J}^k(\NNskel) \in \SetSaddle_J \}$ have zero Lebesgue measure.
Secondly, we need to apply Lemma \ref{preserving_measure_zero_sets} to $f_{\stepsize}$ to conclude with the previous lemma that $\{ \NNskel \in \R^d \colon \lim_{k \rightarrow \infty} f_{\stepsize}^k(\NNskel) \in \SetSaddle \}$ also has zero Lebesgue measure.


\subsection{Non-degeneracy almost everywhere}
\label{section_nondegenerate}

In this section, we show that there exists an open subset of $\R^d$ of full measure such that the modified dynamical system exhibits the regularity required by Proposition \ref{GD_basin_measure_zero} on that subset.
For any $J = (J_+,J_-) \in \J$, let $U_0^J \subseteq \R^d$ be the set of all networks without degenerate neurons such that $w_j \ne 0$ for all $j \in J_+ \cup J_-$; let $U_1^J \subseteq \R^d$ be the set of all networks without degenerate neurons such that no neuron in $\{1,\dots,N\} \backslash (J_+ \cup J_-)$ is semi-inactive or type-1-active with breakpoint $\timeZero$ or $\timeOne$; and let $U^J = U_0^J \cap U_1^J$.
Let $U_{\infty} \subseteq U^{\emptyset,\emptyset} \cap U^{\{1,\dots,N\},\emptyset}$ be the set of all networks that do not have two distinct type-2-active neurons with the same breakpoint.
We remark that $U_{\infty} \subseteq \R^d$ is open and has full measure.

\begin{lemma}
\label{regularity_modified_gradient}
	Let $J = (J_+,J_-) \in \J$.
Then, the following properties hold.
\begin{enumerate}[\rmfamily (i)]\itemsep = 0em
\item $\Grad^J$ is continuously differentiable on $U^J$.

\item The Jacobian $(\Grad^J)'(\NNskel)$ is a symmetric matrix for all $\NNskel \in U^J$ for which for all $\tau \in \{+,-\}$ and $j \in J_{\tau}$ the $j^{th}$ neuron of $\NNskel$ is semi-inactive with $\mathrm{sign}(w_j) = \tau$.

\item If $\target$ is Lipschitz continuous, then the Jacobian of $\Grad^J$ is locally Lipschitz continuous on $U^J$.

\item If $\target$ is a polynomial, then $\Grad^J$ is a rational function\footnote{More precisely, $\Grad^J(\NNskel)$ equals $p(\NNskel)/q(\NNskel)$ for two polynomials $p$ and $q$ on each connected component but $p$ and $q$ could be different polynomials for different components.} on $U_{\infty}$.
\end{enumerate}
\end{lemma}

\begin{proof}
	The set $U_0^{(\emptyset,\emptyset)}$ is the set of all networks without degenerate neurons.
For all $j \in \{1,\dots,N\}$, we let $r_j,s_j \colon U_0^{(\emptyset,\emptyset)} \rightarrow \R$ be the functions given by
\begin{equation*}
	r_j(\NNskel) =
	\begin{cases}
		\frac{\timeZero+\timeOne}{2} - \frac{w_j(\timeOne-\timeZero)^2}{2w_j(\timeZero+\timeOne)+4b_j} &\text{if the } j^{th} \text{ neuron of } \NNskel \text{ is inactive}, \\
		\timeOne &\text{if the } j^{th} \text{ neuron of } \NNskel \text{ is semi-inactive with } w_j > 0, \\
		t_j &\text{if the } j^{th} \text{ neuron of } \NNskel \text{ is type-2-active with } w_j > 0, \\
		\timeZero &\text{otherwise}
	\end{cases}
\end{equation*}%
and
\begin{equation*}
	s_j(\NNskel) =
	\begin{cases}
		r_j(\NNskel) &\text{if the } j^{th} \text{ neuron of } \NNskel \text{ is inactive}, \\
		\timeZero &\text{if the } j^{th} \text{ neuron of } \NNskel \text{ is semi-inactive with } w_j < 0, \\
		t_j &\text{if the } j^{th} \text{ neuron of } \NNskel \text{ is type-2-active with } w_j < 0, \\
		\timeOne &\text{otherwise}.
	\end{cases}
\end{equation*}%
$r_j(\NNskel)$ and $s_j(\NNskel)$ are the endpoints of the interval $I_j$ if the $j^{th}$ neuron of $\NNskel$ is not inactive and $[r_j,s_j]$ is a singleton if it is inactive.
Observe that $r_j$ and $s_j$ are locally Lipschitz continuous and, for any connected component $V$ of $U_1^{(\{j\}^c,\emptyset)}$, the restrictions $r_j|_V$ and $s_j|_V$ are rational functions.
In particular, $r_j$ and $s_j$ are infinitely often differentiable on $U_1^{(\{j\}^c,\emptyset)}$.
Next, we define similar functions $r_j^J,s_j^J \colon U_0^J \rightarrow \R$ by
\begin{equation*}
	r_j^J(\NNskel) =
	\begin{cases}
		r_j(\NNskel) &\text{if } j \notin J_+ \cup J_-, \\
		t_j &\text{if } j \in J_+, \\
		\timeZero &\text{if } j \in J_-,
	\end{cases}
	\qquad
	s_j^J(\NNskel) =
	\begin{cases}
		s_j(\NNskel) &\text{if } j \notin J_+ \cup J_-, \\
		\timeOne &\text{if } j \in J_+, \\
		t_j &\text{if } j \in J_-.
	\end{cases}
\end{equation*}%
These functions are locally Lipschitz continuous on $U_0^J$ and infinitely often differentiable on $U^J$ because if $j \notin J_+ \cup J_-$, then $U^J \subseteq U_1^{(\{j\}^c,\emptyset)}$.
Now, for all $\NNskel \in U_0^J$, $j \in \{1,\dots,N\}$, $i \in \{0,1,2\}$,
\begin{equation*}
\begin{split}
	\Grad^J_{iN+j}(\NNskel) &= 2 \int_{r_j^J}^{s_j^J} \Big( \frac{\partial}{\partial \NNskel_{iN+j}} v_j(w_jx+b_j) \Big) (\NNfct(x)-\target(x)) \, dx, \\
	\Grad^J_{3N+1}(\NNskel) &= 2 \int_{\timeZero}^{\timeOne} (\NNfct(x)-\target(x)) \, dx.
\end{split}
\end{equation*}%
Thus, all partial derivatives of $\Grad^J$ exist on $U^J$ by the Leibniz integral rule and are given by
\begin{equation*}
\begin{split}
	\frac{\partial}{\partial \NNskel_{i'N+j'}} \Grad^J_{iN+j}(\NNskel)
	&= 2 \int_{r_j^J}^{s_j^J} \Big( \frac{\partial}{\partial \NNskel_{iN+j}} v_j(w_jx+b_j) \Big) \Big( \frac{\partial}{\partial \NNskel_{i'N+j'}} v_{j'}(w_{j'}x+b_{j'}) \Big) \IndFct{[r_{j'},s_{j'}]}(x) \, dx \\
	&\quad+ 2 \int_{r_j^J}^{s_j^J} \Big( \frac{\partial}{\partial \NNskel_{i'N+j'}} \frac{\partial}{\partial \NNskel_{iN+j}} v_j(w_jx+b_j) \Big) (\NNfct(x)-\target(x)) \, dx \\
	&\quad+ 2 \Big( \frac{\partial}{\partial \NNskel_{iN+j}} v_j(w_jx+b_j) \Big) (\NNfct(x)-\target(x))\Big|_{x = s_j^J} \Big( \frac{\partial}{\partial \NNskel_{i'N+j'}} s_j^J \Big) \\
	&\quad- 2 \Big( \frac{\partial}{\partial \NNskel_{iN+j}} v_j(w_jx+b_j) \Big) (\NNfct(x)-\target(x))\Big|_{x = r_j^J} \Big( \frac{\partial}{\partial \NNskel_{i'N+j'}} r_j^J \Big)
\end{split}
\end{equation*}
and
\begin{equation*}
\begin{split}
	\frac{\partial}{\partial \NNskel_{3N+1}} \Grad^J_{iN+j}(\NNskel) &= 2 \int_{r_j^J}^{s_j^J} \frac{\partial}{\partial \NNskel_{iN+j}} v_j(w_jx+b_j) \, dx, \\
	\frac{\partial}{\partial \NNskel_{iN+j}} \Grad^J_{3N+1}(\NNskel) &= 2 \int_{r_j}^{s_j} \frac{\partial}{\partial \NNskel_{iN+j}} v_j(w_jx+b_j) \, dx, \\
	\frac{\partial}{\partial \NNskel_{3N+1}} \Grad^J_{3N+1}(\NNskel) &= 2.
\end{split}
\end{equation*}%
In particular, all partial derivatives of $\Grad^J$ are continuous and, hence, $\Grad^J$ is continuously differentiable on $U^J$.
This proves (i).
Moreover, since $r_j$, $s_j$, $r_j^J$, $s_j^J$, and $\NNskel \mapsto \NNfct$ are locally Lipschitz continuous on $U_0^J$, it follows from the above formulas that if $\target$ is Lipschitz, then $(\Grad^J)'$ is locally Lipschitz on $U^J$.
Next, note the following equality, for all $j,j' \in \{1,\dots,N\}$, $i,i' \in \{0,1,2\}$, $\NNskel \in U_0^{(\{j,j'\},\emptyset)}$;
\begin{equation*}
	\Big( \frac{\partial}{\partial \NNskel_{iN+j}} v_j(w_jx+b_j) \Big)\Big|_{x = t_j} \frac{\partial}{\partial \NNskel_{i'N+j'}} t_j = \Big( \frac{\partial}{\partial \NNskel_{i'N+j'}} v_{j'}(w_{j'}x+b_{j'}) \Big)\Big|_{x = t_{j'}} \frac{\partial}{\partial \NNskel_{iN+j}} t_{j'}.
\end{equation*}%
Therefore, if $\NNskel \in U^J$ satisfies $r_j^J = r_ j$ and $s_j^J = s_j$ for all $j \in \{1,\dots,N\}$, then $(\Grad^J)'(\NNskel)$ is symmetric.
In particular, this holds for all $\NNskel \in U^J$ satisfying the conditions of (ii).

Now, suppose $\target$ is a polynomial.
For any $\NNskel \in U_0^J$, $j \in \{1,\dots,N\}$, $i \in \{0,1,2\}$, we can write
\begin{equation*}
\begin{split}
	\Grad^J_{iN+j}(\NNskel) &= 2 \int_{r_j^J}^{s_j^J} \Big( \frac{\partial}{\partial \NNskel_{iN+j}} v_j(w_jx+b_j) \Big) (c-\target(x)) \, dx \\
	&\quad + 2 \sum_{n=1}^{N} \int_{r_j^J}^{s_j^J} \Big( \frac{\partial}{\partial \NNskel_{iN+j}} v_j(w_jx+b_j) \Big) v_n(w_nx+b_n) \IndFct{[r_n,s_n]}(x) \, dx.
\end{split}
\end{equation*}%
The functions $P_{i,j} \colon \R^d \times [\timeZero,\timeOne] \rightarrow \R$ given by
\begin{equation*}
\begin{split}
	P_{i,j}(\NNskel,x) &= \int_0^x \Big( \frac{\partial}{\partial \NNskel_{iN+j}} v_j(w_jy+b_j) \Big) (c-\target(y)) \, dy \\
	&= \frac{\partial}{\partial \NNskel_{iN+j}} \left( \frac{1}{2}v_jw_jcx^2 + v_jb_jcx - v_j(w_jx+b_j) \int_0^x \target(y) \, dy + v_jw_j \int_0^x \int_0^y \target(z) \, dzdy \right)
\end{split}
\end{equation*}%
are polynomials.
By definition of these functions, for any $\NNskel \in U_0^J$,
\begin{equation*}
	\int_{r_j^J}^{s_j^J} \Big( \frac{\partial}{\partial \NNskel_{iN+j}} v_j(w_jx+b_j) \Big) (c-\target(x)) \, dx = P_{i,j}(\NNskel,s_j^J) - P_{i,j}(\NNskel,r_j^J).
\end{equation*}%
For any $j \in \{1,\dots,N\}$ and any connected component $V$ of $U^{\emptyset,\emptyset} \cap U^J$, the functions $r_j^J$ and $s_j^J$ equal each other, are both constant, or one is constant and one equals $t_j = -b_j/w_j$ throughout that entire component.
In particular, one of the following four cases holds for all $\NNskel \in V$:
\begin{equation*}
\begin{split}
	P_{i,j}(\NNskel,s_j^J) - P_{i,j}(\NNskel,r_j^J)
	=
	\begin{cases}
	P_{i,j}(\NNskel,\timeOne) - P_{i,j}(\NNskel,\timeZero), \\
	P_{i,j}(\NNskel,t_j) - P_{i,j}(\NNskel,\timeZero), \\
	P_{i,j}(\NNskel,\timeOne) - P_{i,j}(\NNskel,t_j), \\
	0.
	\end{cases}
\end{split}
\end{equation*}%
Since $U^{\emptyset,\emptyset} \cap U^J$ has only finitely many connected components, it follows that we can take $q \in \N$ sufficiently large so that
\begin{equation*}
	\NNskel \mapsto \left( \prod_{k=1}^N w_k^q \right) \left( P_{i,j}(\NNskel,s_j^J) - P_{i,j}(\NNskel,r_j^J) \right)
\end{equation*}%
is a polynomial on $U^{\emptyset,\emptyset}$.
The remainder of the proof is similar to the previous step.
The functions $P_{i,j,n} \colon \R^d \times [\timeZero,\timeOne] \rightarrow \R$ given by
\begin{equation*}
\begin{split}
	P_{i,j,n}(\NNskel,x) &= \int_0^x \Big( \frac{\partial}{\partial \NNskel_{iN+j}} v_j(w_jy+b_j) \Big) v_n(w_ny+b_n) \, dy \\
	&= \frac{1}{6} \Big( \frac{\partial}{\partial \NNskel_{iN+j}} v_jw_j \Big) ( 2v_nw_nx^3 + 3v_nb_nx^2 ) + \frac{1}{2} \Big( \frac{\partial}{\partial \NNskel_{iN+j}} v_jb_j \Big) ( v_nw_nx^2 + 2v_nb_nx )
\end{split}
\end{equation*}%
are polynomials.
Given any $j,n \in \{1,\dots,N\}$ and any connected component $V$ of $U^{\emptyset,\emptyset} \cap U^J$, if the $j^{th}$ and $n^{th}$ neuron are not both type-2-active for some (and, hence, any) network in $V$, then one of the following three cases holds for all $\NNskel \in V$:
\begin{equation*}
\begin{split}
	\int_{r_j^J}^{s_j^J} \Big( \frac{\partial}{\partial \NNskel_{iN+j}} v_j(w_jx+b_j) \Big) v_n(w_nx+b_n) \IndFct{[r_n,s_n]}(x) \, dx
	=
	\begin{cases}
	P_{i,j,n}(\NNskel,s_j^J) - P_{i,j,n}(\NNskel,r_j^J), \\
	P_{i,j,n}(\NNskel,s_n) - P_{i,j,n}(\NNskel,r_n), \\
	0.
	\end{cases}
\end{split}
\end{equation*}%
Now, suppose the $j^{th}$ and $n^{th}$ neuron are both type-2-active.
Since the definition of $U_{\infty}$ excludes that the breakpoints $t_j$ and $t_n$ cross each other, one of the following five cases holds throughout a given connected component of $U_{\infty}$:
\begin{equation*}
\begin{split}
	\int_{r_j^J}^{s_j^J} \Big( \frac{\partial}{\partial \NNskel_{iN+j}} v_j(w_jx+b_j) \Big) v_n(w_nx+b_n) \IndFct{[r_n,s_n]}(x) \, dx
	=
	\begin{cases}
	P_{i,j,n}(\NNskel,s_j^J) - P_{i,j,n}(\NNskel,r_n), \\
	P_{i,j,n}(\NNskel,s_j^J) - P_{i,j,n}(\NNskel,r_j^J), \\
	P_{i,j,n}(\NNskel,s_n) - P_{i,j,n}(\NNskel,r_j^J), \\
	P_{i,j,n}(\NNskel,s_n) - P_{i,j,n}(\NNskel,r_n), \\
	0.
	\end{cases}
\end{split}
\end{equation*}%
This implies that
\begin{equation*}
	\NNskel \mapsto \left( \prod_{k=1}^N w_k^q \right) \int_{r_j^J}^{s_j^J} \Big( \frac{\partial}{\partial \NNskel_{iN+j}} v_j(w_jx+b_j) \Big) v_n(w_nx+b_n) \IndFct{[r_n,s_n]}(x) \, dx
\end{equation*}%
is a polynomial on $U_{\infty}$ for a sufficiently large $q \in \N$.
We conclude that also $\NNskel \mapsto \Grad_{iN+j}^J(\NNskel) \prod_{k=1}^N w_k^q$ is a polynomial on $U_{\infty}$ for a sufficiently large $q \in \N$.
The same argument in a simplified version works for $\Grad^J_{3N+1}(\NNskel) = 2 \int_{\timeZero}^{\timeOne} (\NNfct(x)-\target(x)) \, dx$.
\end{proof}

The last bit of regularity we need to check is the non-degeneracy of the Jacobian $f_{\stepsize,J}'$.
Lemma \ref{regularity_modified_gradient}.(iv) enables this.
The proof of the next lemma uses the argument sketched in Section \ref{section_applications}.

\begin{lemma}
\label{nondegenerate_Jacobian}
	If $\target$ is a polynomial, then for almost all $\stepsize \in (0,\infty)$ and for all $J \in \J$ there exists an open set $U^J_{\stepsize} \subseteq U_{\infty}$ of full measure such that $\det(f_{\stepsize,J}'(\NNskel)) \ne 0$ for all $\NNskel \in U^J_{\stepsize}$.
\end{lemma}

\begin{proof}
	Fix $J \in \J$.
Since $\Grad^J$ is a rational function on $U_{\infty}$, there exists a polynomial\footnote{Here, $p$ is a polynomial on each connected component but need not be given by the same formula on each component.} $p \colon U_{\infty} \rightarrow \R$, which is not constantly zero on any connected component of $U_{\infty}$, such that $\NNskel \mapsto p(\NNskel) \Grad^J(\NNskel)$ is a polynomial on $U_{\infty}$.
Since the derivative of a polynomial is still a polynomial, it follows that
\begin{equation*}
	\NNskel \mapsto p(\NNskel)^2 (\Grad^J)'(\NNskel) = (pp\Grad^J)'(\NNskel) - 2 p(\NNskel) p'(\NNskel) \Grad^J(\NNskel)
\end{equation*}%
is also a polynomial on $U_{\infty}$.
The differential of $f_{\stepsize,J}$ on $U_{\infty}$ is $f_{\stepsize,J}'(\NNskel) = I - \stepsize (\Grad^J)'(\NNskel)$.
Therefore, the map $P \colon U_{\infty} \! \times \R \rightarrow \R$ given by $P(\NNskel,\stepsize) = \det(p(\NNskel)^2 f_{\stepsize,J}'(\NNskel))$ is a polynomial.
Moreover, $P$ is not constantly zero on any connected component of $U_{\infty} \! \times \R$ because $P(\NNskel,0) = p(\NNskel)^2$.
In particular, its zero set $P^{-1}(0)$ has Lebesgue measure zero.
For every $\stepsize \in \R$, denote $Z_{\stepsize} = \{\NNskel \in U_{\infty} \colon P(\NNskel,\stepsize) = 0\}$.
By Tonelli's theorem,
\begin{equation*}
	0 = \int_{P^{-1}(0)} d\NNskel d\stepsize = \int_{\R} \int_{Z_{\stepsize}} d\NNskel d\stepsize,
\end{equation*}%
from which it follows that $Z_{\stepsize}$ has zero Lebesgue measure for almost every $\stepsize \in \R$.
Set $U_{\stepsize}^J = U_{\infty} \backslash Z_{\stepsize}$.
\end{proof}

This concludes the discussion of the regularity requirements.
It remains to establish strictness of saddle points of $\Loss$.


\subsection{Strict saddle points}
\label{section_strictness}

To investigate saddle points of $\Loss$, it is useful to classify them in terms of their types of neurons.
The next result follows from \cite[Theorem 2.4 and Corollary 2.7]{CheJenRos2022JNLS}.

\begin{proposition}
\label{classification}
	Assume $\target \colon [\timeZero,\timeOne] \rightarrow \R$ is affine but not constant and let $\NNskel = (w,b,v,c) \in U_0^{(\emptyset,\emptyset)}$ be a critical point of $\Loss$ that is not a global minimum.
Then, the following hold:
\begin{enumerate}[\rmfamily (i)]\itemsep = 0em

\item $\NNskel$ is not a local maximum of $\Loss$.

\item $\NNskel$ is a local minimum of $\Loss$ if and only if $c = \target(\frac{\timeZero+\timeOne}{2})$ and, for all $j \in \{1,\dots,N\}$, the $j^{th}$ hidden neuron of $\NNskel$ is inactive or semi-inactive with $\target'(\timeZero) v_j w_j < 0$.

\item $\NNskel$ is a saddle point of $\Loss$ if and only if $c = \target(\frac{\timeZero+\timeOne}{2})$, $\NNskel$ does not have any type-1-active or non-flat semi-active neurons, and exactly one of the following two conditions holds:
\begin{enumerate}[\rmfamily (a)]\itemsep = 0em

\item $\NNskel$ does not have any type-2-active neurons and there exists $j \in \{1,\dots,N\}$ such that the $j^{th}$ hidden neuron of $\NNskel$ is flat semi-active or semi-inactive with $\target'(\timeZero) v_j w_j \geq 0$.

\item There exists $n \in \{2,4,6,\dots\}$ such that $(\bigcup_{j \in \{1,\dots,N\},\, w_j \ne 0} \{-\frac{b_j}{w_j}\}) \cap (\timeZero,\timeOne) = \bigcup_{i=1}^{n} \{\timeZero+\frac{i(\timeOne-\timeZero)}{n+1}\}$ and, for all $j \in \{1,\dots,N\}$, $i \in \{1,\dots,n\}$ with $w_j \ne 0 = b_j + w_j(\timeZero+\frac{i(\timeOne-\timeZero)}{n+1})$, it holds that $\mathrm{sign}(w_j) = (-1)^{i+1}$ and $\sum_{k \in \{1,\dots,N\},\, w_k \ne 0 = b_k + w_k(\timeZero+\frac{i(\timeOne-\timeZero)}{n+1})}{} v_k w_k = \frac{2\target'(\timeZero)}{n+1}$.

\end{enumerate}

\item There exists $n \in \{0,2,4,\dots\}$ with $n \leq N$ such that
$\displaystyle{
\Loss(\NNskel) = \frac{[\target'(\timeZero)]^2 (\timeOne-\timeZero)^3}{12(n+1)^4}
}$ and
\begin{equation*}
	\NNfct(x) = \target(x) - \frac{(-1)^i \target'(\timeZero)}{n+1} \Big( x - \timeZero - \frac{(i+\frac{1}{2})(\timeOne-\timeZero)}{n+1} \Big)
\end{equation*}%
for all $i \in \{0,\dots,n\}$, $x \in [\timeZero+\frac{i(\timeOne-\timeZero)}{n+1},\timeZero+\frac{(i+1)(\timeOne-\timeZero)}{n+1}]$.

\end{enumerate}
\end{proposition}

Now, we can clarify for which saddle points of $\Loss$ we can establish strictness.
For an affine target function $\target$, let $\SetSaddle \subseteq U_0^{(\emptyset,\emptyset)}$ be the set of all saddle points of $\Loss$ that are not solely comprised of inactive neurons and semi-inactive neurons with $\target'(\timeZero)v_jw_j \leq 0$.
As in Lemma \ref{GD_regularity_increase}, for all $J \in \J$, let $\SetSaddle_J \subseteq \SetSaddle$ be the set of networks $\NNskel$ such that $J_+$ is exactly the set of neurons of $\NNskel$ that are semi-inactive with $w_j>0$ and $J_-$ is exactly the set of neurons of $\NNskel$ that are semi-inactive with $w_j<0$.

Recall that $f_{\stepsize,J}(\NNskel) = \NNskel - \stepsize \Grad^J(\NNskel)$.
Thus, to show that $f_{\stepsize,J}'(\NNskel) = I - \stepsize (\Grad^J)'(\NNskel)$ has an eigenvalue of absolute value strictly greater than 1, it is sufficient to show that $(\Grad^J)'(\NNskel)$ has a strictly negative eigenvalue.

\begin{lemma}
\label{saddle_points_are_strict}
	Assume $\target$ is affine but not constant, and let $J \in \J$, $\NNskel \in \SetSaddle_J$.
Then, $\Grad^J(\NNskel) = 0$ and the matrix $(\Grad^J)'(\NNskel)$ has a strictly negative eigenvalue.
\end{lemma}

\begin{proof}
	On the one hand, for all $j \notin J_+ \cup J_-$ and $i \in \{0,1,2\}$, we know that $\Grad^J_{iN+j}(\NNskel) = \Grad_{iN+j}(\NNskel) = 0$.
On the other hand, for all $j \in J_+$, we have that $t_j = \timeOne$ and, hence, also $\Grad^J_{iN+j}(\NNskel) = 0$ for all $i \in \{0,1,2\}$; likewise for $j \in J_-$.
This shows that $\Grad^J(\NNskel) = 0$.

Proposition \ref{classification} tells us that $\NNskel$ has no type-1-active neurons, so $\NNskel \in U^J$ and $\Grad^J$ is differentiable at $\NNskel$ with symmetric Jacobian $(\Grad^J)'(\NNskel)$ by Lemma \ref{regularity_modified_gradient}.
We will conclude the proof by showing that $(\Grad^J)'(\NNskel)$ contains a strictly negative principle minor.
To this end, we distinguish two cases.
First, if $\NNfct$ is affine on $[\timeZero,\timeOne]$, then $\NNskel$ must have a flat semi-active neuron or a semi-inactive neuron with $\target'(\timeZero) v_jw_j>0$, by Proposition \ref{classification} and by the definition of the set $\SetSaddle$.
If $\NNskel$ has a flat semi-active neuron $j$, then
\begin{equation*}
\begin{split}
	\det
	\begin{pmatrix}
	\frac{\partial}{\partial \NNskel_j} \Grad^J_j(\NNskel) & \frac{\partial}{\partial \NNskel_j} \Grad^J_{2N+j}(\NNskel) \\
	\frac{\partial}{\partial \NNskel_{2N+j}} \Grad^J_j(\NNskel) & \frac{\partial}{\partial \NNskel_{2N+j}} \Grad^J_{2N+j}(\NNskel)
	\end{pmatrix}
	&=
	\det
	\begin{pmatrix}
	0 & -\frac{1}{6} \target'(\timeZero) (\timeOne-\timeZero)^3 \\
	-\frac{1}{6} \target'(\timeZero) (\timeOne-\timeZero)^3 & 2b_j^2 (\timeOne-\timeZero)
	\end{pmatrix}
	\\
	&= -\frac{1}{36} [\target'(\timeZero)]^2 (\timeOne-\timeZero)^6 < 0.
\end{split}
\end{equation*}%
If $\NNskel$ has a semi-inactive neuron $j$ with $\target'(\timeZero) v_jw_j > 0$, then
\begin{equation*}
	\frac{\partial}{\partial \NNskel_{N+j}} \Grad^J_{N+j}(\NNskel) = - \target'(\timeZero) \frac{v_j}{w_j} (\timeOne-\timeZero) < 0.
\end{equation*}%
Secondly, if $\NNfct$ is not affine on $[\timeZero,\timeOne]$, then exactly as in the proof
of \cite[Lemma 2.24]{CheJenRos2022JNLS} we can find a set of coordinates corresponding to type-2-active neurons such that the determinant of the Hessian $H$ of $\Loss$ restricted to these coordinates is strictly negative.
Since this involves only neurons in $\{1,\dots,N\} \backslash (J_+ \cup J_-)$, the matrix $(\Grad^J)'(\NNskel)$ contains $H$ as a submatrix.
\end{proof}

Having established strictness of saddle points, it is now straight-forward to apply Proposition \ref{GD_basin_measure_zero}, which yields the following result.

\begin{theorem}
\label{GD_avoids_saddle_points}
	Assume $\target$ is affine but not constant.
Then, for almost every stepsize $\stepsize \in (0,\infty)$, the set $\{\NNskel \in \R^d \colon \lim_{k \rightarrow \infty} f_{\stepsize}^k(\NNskel) \in \SetSaddle \}$ has Lebesgue measure zero.
\end{theorem}

\begin{proof}
	By Lemmas
\ref{regularity_modified_gradient}, \ref{nondegenerate_Jacobian}, and \ref{saddle_points_are_strict},
we can apply Proposition \ref{GD_basin_measure_zero} for almost all $\stepsize \in (0,\infty)$ and all $J \in \J$ to the dynamical system $f_{\stepsize,J}$ and the set $\SetSaddle_J$ with $U=U^J$ and $V = U^J_{\stepsize}$ to find that $\{\NNskel \in \R^d \colon \lim_{k \rightarrow \infty} f_{\stepsize,J}^k(\NNskel) \in \SetSaddle_J \}$ has Lebesgue measure zero.
Since $\Grad = \Grad^{(\emptyset,\emptyset)}$, Lemmas
\ref{regularity_modified_gradient} and \ref{nondegenerate_Jacobian} enable us to apply Lemma \ref{preserving_measure_zero_sets} to $f_{\stepsize}$ so that, together with Lemma \ref{GD_regularity_increase}, we obtain the desired result.
\end{proof}


\subsection{Convergence to global minima for suitable initialization}
\label{section_convergence}

Suppose a trajectory of gradient descent for the loss function $\Loss$ with affine nonconstant target function $\target$ converges to a critical point of $\Loss$.
If the gradient descent algorithm was initialized randomly under a probability measure that is absolutely continuous with respect to the Lebesgue measure, then, with probability one, the limit critical point is not a saddle point in $\SetSaddle$ by Theorem \ref{GD_avoids_saddle_points}.
Here, $\SetSaddle$ is the same set of saddle points as specified above Lemma \ref{saddle_points_are_strict}.
We can say more about the limit critical point using Proposition \ref{classification}.(iv).
It states that there are only finitely many possibilities for the value of the loss function at its critical points, which we can think of as partitioning the set of all critical points into ``layers''.
In particular, if the loss at the limit critical point is below the threshold $\frac{[\target'(\timeZero)]^2 (\timeOne-\timeZero)^3}{12(N+1)^4}$, then this critical point must belong to the first layer, that is it must be a global minimum.
We can improve this threshold to include higher layers if we verify that the limit critical point could only be a global minimum or belong to $\SetSaddle$, which entails the absence of degenerate neurons.
If $N$ is even, then networks in the second layer of critical points consist of only type-2-active neurons, which forces the limit critical point to belong to $\SetSaddle$ without additional hypotheses.

\begin{proposition}
\label{conv_to_global_min}
	Assume $\target \colon [\timeZero,\timeOne] \rightarrow \R$ is affine but not constant.
Let $\ConvTraj$ denote the set $\ConvTraj = \{ \NNskel \in \R^d \colon (f_{\stepsize}^k(\NNskel))_{k \in \N_0} \text{ is convergent} \}$.
Then, the following hold:
\begin{enumerate}[\rmfamily (i)]\itemsep = 0em
\item
For almost all $\stepsize \in (0,\infty)$ and almost all
\begin{equation}
\label{conv_to_global_min_set}
	\NNskel \in \left\{ \NNskelAlt \in \ConvTraj \colon \lim_{k \rightarrow \infty} f_{\stepsize}^k(\NNskelAlt) \text{ has no degenerate neurons and } \lim_{k \rightarrow \infty} \Loss\left(f_{\stepsize}^k(\NNskelAlt)\right) < \textstyle{\frac{[\target'(\timeZero)]^2 (\timeOne-\timeZero)^3}{12}} \right\},
\end{equation}%
it holds that $\lim_{k \rightarrow \infty} \Loss(f_{\stepsize}^k(\NNskel)) = 0$.

\item
If $N$ is even, then for all $\NNskel \in \ConvTraj$ with $0 < \lim_{k \rightarrow \infty} \Loss\left(f_{\stepsize}^k(\NNskel)\right) < \textstyle{\frac{[\target'(\timeZero)]^2 (\timeOne-\timeZero)^3}{12(N-1)^4}}$ the network $\lim_{k \rightarrow \infty} f_{\stepsize}^k(\NNskel)$ has no degenerate neurons.
\end{enumerate}
\end{proposition}

\begin{proof}
	By definition of $f_{\stepsize}$, for any $\NNskel_0 \in \ConvTraj$,
\begin{equation}
\label{pf_conv_to_global_min}
	\lim_{k \rightarrow \infty} \norm{\Grad(f_{\stepsize}^k(\NNskel_0))} = \lim_{k \rightarrow \infty} \frac{1}{\stepsize} \norm{f_{\stepsize}^k(\NNskel_0)-f_{\stepsize}^{k+1}(\NNskel_0)} = 0.
\end{equation}%
Now, let $\NNskel_0$ belong to the set specified in \eqref{conv_to_global_min_set} and let $\NNskel = \lim_{k \rightarrow \infty} f_{\stepsize}^k(\NNskel_0)$.
Then, $\NNskel$ is a point of differentiability of $\Loss$ and $\Grad$ is continuous in a neighborhood of $\NNskel$, so $\NNskel$ is a critical point of $\Loss$.
Further, $\NNfct$ cannot be constant on $[\timeZero,\timeOne]$ since
\begin{equation*}
	\Loss(\NNskel) < \frac{1}{12} [\target'(\timeZero)]^2 (\timeOne-\timeZero)^3 = \inf_{C \in \R} \int_{\timeZero}^{\timeOne} (C-\target(x))^2 \, dx.
\end{equation*}%
In particular, $\NNskel$ cannot consist solely of inactive and semi-inactive neurons.
By Proposition \ref{classification}, $\NNskel$ is a global minimum or a saddle point in $\SetSaddle$.
In the latter case, Theorem \ref{GD_avoids_saddle_points} tells us that $\stepsize$ or $\NNskel_0$ belongs to a set of Lebesgue measure zero.

Next, we show that if $N$ is even, then the property $0 < \Loss(\NNskel) < \textstyle{\frac{[\target'(\timeZero)]^2 (\timeOne-\timeZero)^3}{12(N-1)^4}}$ implies that $\NNskel_0$ belongs to the set in \eqref{conv_to_global_min_set}.
Let $m \in \{0,\dots,N\}$ be the number of degenerate neurons of $\NNskel$.
We show that $m=0$.
Since $\NNfct$ cannot be constant on $[\timeZero,\timeOne]$, we must have $m \leq N-1$.
Let $\NNskelAlt \in \R^{d-3m}$ be the network obtained from $\NNskel$ by dropping its degenerate neurons.
Since the generalized gradient is assumed to agree with the partial derivatives of the loss coordinate-wise when the latter exist, it follows that the generalized gradient defined on $\R^{d-3m}$ is continuous in a neighborhood of $\NNskelAlt$.
This and \eqref{pf_conv_to_global_min} show that $\NNskelAlt$ is a critical point of the loss function defined on $\R^{d-3m}$.
Moreover, since we only removed degenerate neurons, the value of the loss at $\NNskelAlt$ is equal to $\Loss(\NNskel)$.
Proposition \ref{classification} and the assumption on $\Loss(\NNskel)$ imply that $\NNskelAlt$ is a saddle point with $N$ type-2-active neurons.
Thus, $m=0$.
\end{proof}

It would be ideal to get rid of the hypothesis that the limit critical point has no degenerate neurons.
One can deal with this issue heuristically by considering the gradient flow and arguing that gradient descent turns into the gradient flow as the stepsize decreases.
Namely, for the gradient flow it has been observed that trajectories avoid degenerate neurons if the network at initialization obeys the inequality $w_j^2+b_j^2>v_j^2$ for all $j$.
This can be seen by integrating the equality
\begin{equation*}
	w_j \Grad_j(\NNskel) + b_j \Grad_{N+j}(\NNskel) = v_j \Grad_{2N+j}(\NNskel)
\end{equation*}%
along a trajectory of the gradient flow, which reveals that $w_j^2+b_j^2-v_j^2$ is constant in time.
In particular, if $w_j^2+b_j^2>v_j^2$ at the initial network, then $w_j^2+b_j^2 > 0$ at all times, meaning that the $j^{th}$ neuron cannot become degenerate during training.
We refer to \cite{Wojtowytsch2020} and the appendix in \cite{ChizatBach2018} for related discussions.
Furthermore, by this heuristic based on the gradient flow, the threshold in Proposition \ref{conv_to_global_min} on the loss of the limit critical point could instead be imposed on the loss of the initial network since the loss is monotonically decreasing in the continuous time setting.

We remark that a branch of literature has studied network models in the overparametrized regime, where the number $N$ of neurons tends to infinity.
In this regime, the network models exhibit the behavior of a linear model during training;
\cite{AllLiSong2019,DuLeeLiWangZhai2019,
DuZhaiPocSin2019,LiLiang2018,ZouCaoZhouGu2020}.
More specifically, we recover the behavior of a kernel in the infinite width limit;
\cite{EMaWu2020,JacGabHon2018}.
It should be noted that this linearization of the training dynamics requires a specific scaling of the loss (or of the realization of the networks, which has the same effect);
\cite{ChizatOyallonBach2019}.
Without this scaling, networks follow nonlinear dynamics but still converge to a global minimum in the limit $N \rightarrow \infty$;
\cite{ChizatBach2018}.
We stress that our results do not place any requirements on the number of neurons $N$ and, in particular, hold also for small $N$, which is not covered by the aforementioned literature.
Indeed, the arguments in the preceding paragraph show that the appearance of $N$ in the upper bound in Proposition \ref{conv_to_global_min}.(ii) is a technicality and not a requirement for convergence in Proposition \ref{conv_to_global_min}.(i) to take place.
Furthermore, although our definition \eqref{loss_function} of the loss does not incorporate any scaling (nor our definition \eqref{network_realization} of the realization of a network), the results translate verbatim to the setting with scaling.
If we do impose $N$ to be large and we scale the loss, we can go one step further and combine our results with those of \cite{SafranShamir2016}.
They showed that the loss at initialization is smaller and with higher probability, the more neurons are in the network.
Thus, upper bounds as posed in Proposition \ref{conv_to_global_min} are more likely to be satisfied, the larger the networks are.

The fact that we converge to a global minimum in Proposition \ref{conv_to_global_min} has to do with the simplifying assumptions on the optimization problem in this section.
In less specialized frameworks, local minima may appear more frequently and gradient descent may reach a local minimum that is not a global minimum;
\cite{SafranShamir2018}.
Albeit not being a global minimum, many local minima have been shown to achieve a small error, so ending up in such a local minimum is not undesirable in practice;
\cite{IbraJenRie2022,SafranShamir2016}.
To end up in such a local minimum, gradient descent still has to avoid saddle points, theoretical guarantees for which are the contribution of this work.


\section{Gradient descent for shallow leaky ReLU networks}
\label{leaky_section_network_training}

In the previous section, we modified the gradient of the loss and studied the resulting dynamical system as a workaround for regularity issues at saddle points of interest.
In this section, we replace the ReLU activation by leaky ReLU.
This changes the geometry of the loss surface.
Although many saddle points exhibit a similar structure as for ReLU, those saddle points with regularity issues disappear.
Let $\leaky \in (0,1)$ be the parameter of leaky ReLU.
Now, the realization of a network $\NNskel$ is the function $\NNleakyfct \in C(\R,\R)$ given by
\begin{equation*}
	\NNleakyfct(x) = c + \sum_{j=1}^{N} v_j \max\{w_jx+b_j,\leaky (w_jx+b_j)\}.
\end{equation*}%
As before, we denote by $\leakyLoss \in C(\R^d,\R)$ the squared $L^2$-loss with target function $\target \in C([\timeZero,\timeOne],\R)$, that is
\begin{equation*}
	\leakyLoss(\NNskel) = \int_{\timeZero}^{\timeOne} ( \NNleakyfct(x) - \target(x) )^2 \, dx.
\end{equation*}%
Following a reduction trick from \cite{CheJenRos2022JNLS}, let us consider the map $P \colon \R^{3N+1} \rightarrow \R^{6N+1}$ given by $P(w,b,v,c) = (w,-w,b,-b,v,-\leaky v,c)$.
If $\Loss$ denotes the loss in the ReLU case but with $2N$ neurons instead of $N$, then $\leakyLoss = \Loss \circ P$.
In particular, $\leakyLoss$ is differentiable at all coordinates corresponding to non-degenerate or flat degenerate neurons.
Let $\leakyGrad$ be defined analogous to before and consider the gradient descent step $f_{\stepsize,\leaky}(\NNskel) = \NNskel - \stepsize \leakyGrad(\NNskel)$.
This time, there is no need to modify the gradient.
The reason for this lies in the classification of critical points for the leaky ReLU case.
A lack of differentiability of $\leakyGrad$ persists, but semi-inactive neurons, which required us to introduce the modifications $\Grad^J$, do not appear.
Still, we need to verify the hypotheses of Proposition \ref{GD_basin_measure_zero}.
Let $U_0 \subseteq \R^d$ be the set of all networks without degenerate neurons;
let $U_1 \subseteq U_0$ be the set of all networks such that no neuron is semi-inactive or type-1-active with breakpoint $\timeZero$ or $\timeOne$;
and let $U_{\infty} \subseteq U_1$ be the set of all networks that do not have two distinct type-2-active neurons with the same breakpoint and such that $w_j \ne 0$ for all $j \in \{1,\dots,N\}$.

\begin{lemma}
\label{leaky_regularity_modified_gradient}
The following properties hold.
\begin{enumerate}[\rmfamily (i)]\itemsep = 0em
\item $\leakyGrad$ is continuously differentiable on $U_1$.

\item The Jacobian $(\leakyGrad)'(\NNskel)$ is a symmetric matrix for all $\NNskel \in U_1$.

\item If $\target$ is Lipschitz continuous, then the Jacobian of $\leakyGrad$ is locally Lipschitz continuous on $U_1$.

\item If $\target$ is a polynomial, then $\leakyGrad$ is a rational function on $U_{\infty}$.
\end{enumerate}
\end{lemma}

\begin{proof}
	The first and third item follow from Lemma \ref{regularity_modified_gradient} and the fact that $\leakyLoss = \Loss \circ P$.
The second item follows from the first since the Hessian of $\leakyLoss$ exists on $U_1$ and equals the Jacobian of $\leakyGrad$.
The fourth item does not follow from Lemma \ref{regularity_modified_gradient}.(iv) because $P$ does not map the set $U_{\infty}$ as defined above into the set $U_{\infty}$ as defined prior to Lemma \ref{regularity_modified_gradient}.
However, excluding distinct type-2-active neurons with the same breakpoint was needed in the proof of Lemma \ref{regularity_modified_gradient} only to ensure that one of the following cases holds throughout connected components (in the notation of that proof):
\begin{equation*}
\begin{split}
	\int_{r_j^J}^{s_j^J} \Big( \frac{\partial}{\partial \NNskel_{iN+j}} v_j(w_jx+b_j) \Big) v_n(w_nx+b_n) \IndFct{[r_n,s_n]}(x) \, dx
	=
	\begin{cases}
	P_{i,j,n}(\NNskel,s_j^J) - P_{i,j,n}(\NNskel,r_n), \\
	P_{i,j,n}(\NNskel,s_j^J) - P_{i,j,n}(\NNskel,r_j^J), \\
	P_{i,j,n}(\NNskel,s_n) - P_{i,j,n}(\NNskel,r_j^J), \\
	P_{i,j,n}(\NNskel,s_n) - P_{i,j,n}(\NNskel,r_n), \\
	0.
	\end{cases}
\end{split}
\end{equation*}%
Here, we only need to consider $J=(\emptyset,\emptyset)$, in which case $r_j^J = r_j$ and $s_j^J = s_j$.
If the map $P$ is applied to a network with type-2-active neurons, then each such neuron gets doubled with the same breakpoint but with different orientation, that is if $r_j = \timeZero$ and $s_j=t_j$, then $r_{j+N} = t_j$ and $s_{j+N} = \timeOne$, respectively if $r_j = t_j$ and $s_j = \timeOne$, then $r_{j+N} = \timeZero$ and $s_{j+N} = t_j$.
Hence, if the $j^{th}$ neuron is type-2-active and $n=j+N$, then $(r_j,s_j) \cap [r_n,s_n] = \emptyset$ and $P(\NNskel)$ still ends up in the last of the above five cases.
If the $j^{th}$ and the $n^{th}$ neuron of $P(\NNskel)$ are type-2-active but $n \ne j+N$, then the breakpoints $t_j$ and $t_{j+N}$ cannot cross each other in the connected component containing $P(\NNskel)$.
With this addition to the proof, Lemma \ref{regularity_modified_gradient}.(iv) carries over to show that $\leakyGrad$ is a rational function on $U_{\infty}$ if $\target$ is a polynomial.
\end{proof}

The non-degeneracy of the Jacobian $f_{\stepsize,\leaky}'$ follows as before.

\begin{lemma}
\label{leaky_nondegenerate_Jacobian}
	If $\target$ is a polynomial, then for almost all $\stepsize \in (0,\infty)$ there exists an open set $U_{\stepsize} \subseteq U_{\infty}$ of full measure such that $\det((f_{\stepsize,\leaky}'(\NNskel)) \ne 0$ for all $\NNskel \in U_{\stepsize}$.
\end{lemma}

\begin{proof}
	The proof is analogous to the proof of Lemma \ref{nondegenerate_Jacobian}.
\end{proof}

We do not repeat the classification of critical points for $\leakyLoss$ with an affine target function from \cite[Theorem 3.5]{CheJenRos2022JNLS}.
It is similar to Proposition \ref{classification} with the following main differences.
There are no local minima and saddle points cannot have semi-inactive neurons or inactive neurons unless they are flat with $w_j=0$ (morally, because with leaky ReLU as activation these types of neurons are equivalent to type-1-active or semi-active neurons).
Let $\SetSaddle \subseteq \R^d$ be the set of all saddle points of $\leakyLoss$ without degenerate neurons.
Note that $\SetSaddle$ contains more saddle points as it did in the ReLU case.
The next lemma is the analog of Lemma \ref{saddle_points_are_strict}.
The appearance of $\leaky_0$ is a technicality arising from a proof strategy in \cite{CheJenRos2022JNLS}.
As remarked in \cite{CheJenRos2022JNLS}, $\leaky_0$ can likely taken to be 1.

\begin{lemma}
	Assume $\target$ is affine but not constant, and let $\NNskel \in \SetSaddle$.
There exists $\leaky_0 \in (0,1]$ depending only on $N$ such that if $\leaky < \leaky_0$, then the matrix $(\leakyGrad)'(\NNskel)$ has a strictly negative eigenvalue.
\end{lemma}

\begin{proof}
	If $\NNskel$ does not have any type-2-active neurons, then all of its neurons are flat semi-active or flat inactive with $w_j=0$ by \cite[Theorem 3.5]{CheJenRos2022JNLS}.
In particular, $\NNleakyfct$ is constantly equal to $c$ and
\begin{equation*}
	\det
	\begin{pmatrix}
	\frac{\partial}{\partial \NNskel_j} \leakyGrad_j(\NNskel) & \frac{\partial}{\partial \NNskel_j} \leakyGrad_{2N+j}(\NNskel) \\
	\frac{\partial}{\partial \NNskel_{2N+j}} \leakyGrad_j(\NNskel) & \frac{\partial}{\partial \NNskel_{2N+j}} \leakyGrad_{2N+j}(\NNskel)
	\end{pmatrix}
	= -\frac{1}{36} \leaky^{\tau} [\target'(\timeZero)]^2 (\timeOne-\timeZero)^6 < 0,
\end{equation*}%
where $\tau = 1 - \mathrm{sign}(b_j)$.
If $\NNskel$ has a type-2-active neuron, then exactly as in the proof
of \cite[Lemma 3.4]{CheJenRos2022JNLS} we can find $\leaky_0 \in (0,1]$ and a set of coordinates such that the determinant of the Hessian of $\leakyLoss$ restricted to these coordinates is strictly negative provided $\leaky < \leaky_0$.
\end{proof}

The previous three lemmas enable us to apply Proposition \ref{GD_basin_measure_zero} for leaky ReLU networks.

\begin{theorem}
\label{leaky_GD_avoids_saddle_points}
	Assume $\target$ is affine but not constant.
There exists $\leaky_0 \in (0,1]$ depending only on $N$ such that if $\leaky < \leaky_0$, then, for almost every stepsize $\stepsize \in (0,\infty)$, the set $\{\NNskel \in \R^d \colon \lim_{k \rightarrow \infty} f_{\stepsize,\leaky}^k(\NNskel) \in \SetSaddle \}$ has Lebesgue measure zero.
\end{theorem}

Since there are no local minima in the leaky ReLU case, there is no need for a threshold as in Proposition \ref{conv_to_global_min}.(i).

\begin{corollary}
	Assume $\target \colon [\timeZero,\timeOne] \rightarrow \R$ is affine but not constant.
There exists $\leaky_0 \in (0,1]$ depending only on $N$ such that if $\leaky < \leaky_0$, then, for almost all $\stepsize \in (0,\infty)$ and almost all
\begin{equation*}
	\NNskel \in \left\{ \NNskelAlt \in \R^d \colon \lim_{k \rightarrow \infty} f_{\stepsize,\leaky}^k(\NNskelAlt) \text{ exists and has no degenerate neurons} \right\},
\end{equation*}%
it holds that $\lim_{k \rightarrow \infty} \Loss(f_{\stepsize}^k(\NNskel)) = 0$.
\end{corollary}


\section{Proof of the center-stable manifold theorem}
\label{section_proof_center_manifold_thrm}

In this section, we present a proof of Theorem \ref{center_manifold_thrm}.
The structure of the proof follows the appendix of \cite{PanPilWang2019} with some modifications.
We begin with a lemma needed later on.

\subsection{Auxiliary lemma}

The following lemma involves the existence of bump functions on balls of radii $r>0$ with bounds on their derivative independent of $r$.

\begin{lemma}
\label{bump_function}
	For all $r \in (0,\infty)$, there exists $\rho_r \in C^{\infty}(\R^d,\B_r(0))$ with support in $\B_r(0)$, which is the identity on $\B_{r/2}(0)$, such that the Frobenius norm of $\rho_r'$ is uniformly bounded by $6\sqrt{d}$, so, in particular, $\rho_r$ is $6\sqrt{d}$-Lipschitz continuous.
\end{lemma}

\begin{proof}
	This can be achieved, for example, by taking a function $\sigma \in C^{\infty}(\R,[0,1])$ such that $\sigma$ is 1 on $(-\infty,1]$, it is 0 on $[4,\infty)$, and $\sigma'(z) \in [-2/3,0]$ for all $z \in \R$.
A possible choice for $\sigma$ would be $\sigma(z) = e^{3/(z-4)} \big[ e^{3/(z-4)} + e^{3/(1-z)} \big]^{-1}$ for $z \in (1,4)$.
Then, set $\rho_r(x) = x \sigma\big(4\norm{x}^2/r^2\big)$.
We estimate the square of the Frobenius norm of $d\rho_r$ by
\begin{equation*}
\begin{split}
	\sum_{j,k=1}^d \Big[ \frac{\partial}{\partial x_k} \rho_r(x)_j \Big]^2
	&= d \Big[\! \underbrace{\sigma\Big(\frac{4\norm{x}^2}{r^2}\Big)}_{\leq 1} \!\Big]^2 + \frac{16 \norm{x}^2}{r^2} \sigma\Big(\frac{4\norm{x}^2}{r^2}\Big) \underbrace{\sigma'\Big(\frac{4\norm{x}^2}{r^2}\Big)}_{\leq 0} + \frac{64 \norm{x}^4}{r^4} \underbrace{\Big[ \sigma'\Big(\frac{4\norm{x}^2}{r^2}\Big) \Big]^2}_{\leq 4/9} \\
	&\leq d + \frac{256}{9} < 36 d.
\end{split}
\end{equation*}%
\end{proof}


\subsection{Proof of the theorem in the diagonal case}

In this section, we will proof Theorem \ref{center_manifold_thrm} in a special case.
Denote $s = \dim(E^{cs}_z) \in \{0,\dots,d-1\}$ and $A = f'(z) \in \R^{d \times d}$.
Assume that $z=0$ and that $A$ is a diagonal matrix with diagonal entries $\lambda_1,\dots,\lambda_d$ such that $\lambda_1,\dots,\lambda_s \in [-1,1]$ and $\lambda_{s+1},\dots,\lambda_d \in \R \backslash [-1,1]$.
We deduce the general case in the next section.
Denote by $B \in \R^{s \times s}$ the diagonal matrix with diagonal entries $\lambda_1,\dots,\lambda_s$ and by $C \in \R^{(d-s) \times (d-s)}$ the diagonal matrix with diagonal entries $\lambda_{s+1},\dots,\lambda_d$.
Denote projections $\Pi^+ \colon \R^d \rightarrow E^{cs} = E^{cs}_z$ and $\Pi^- \colon \R^d \rightarrow E^u = E^u_z$ onto the first $s$ coordinates and onto the last $d-s$ coordinates, respectively.
For any $x \in \R^d$, we write $x^+ = \Pi^+(x)$ and $x^- = \Pi^-(x)$ so that $x = (x^+,x^-)$.
Similarly, we write $g^+ = \Pi^+ \circ g$ and $g^- = \Pi^- \circ g$ for any function $g \colon \R^d \rightarrow \R^d$ so that $g(x) = (g^+(x),g^-(x))$.
Note that
\begin{equation*}
	\max\big\{\norm{x^+},\norm{x^-}\big\} \leq \norm{x} \leq \norm{x^+} + \norm{x^-}.
\end{equation*}%
We use the following convention throughout this proof: we denote by $A^0 \in \R^{d \times d}$ and $B^0 \in \R^{s \times s}$ identity matrices even if one of the entries of the matrices $A$ and $B$ is zero.
The matrices $A^j \in \R^{d \times d}$, $j \in \N_0$, split into a center-stable and an unstable component.
More precisely, they take on the block form
\begin{equation*}
	A^j =
	\begin{pmatrix}
	B^j & 0 \\
	0 & C^j
	\end{pmatrix}
	\colon E^{cs} \oplus E^u \rightarrow E^{cs} \oplus E^u, \quad x \mapsto (B^j x^+,C^j x^-).
\end{equation*}%
Denote by $\eta \colon \R^d \rightarrow \R^d$ the remainder term of the first-order Taylor expansion of $f$ around $0$, that is the map $\eta(x) = f(x)-A x$.
The dynamical system is given for all $k \in \N_0$, $x \in \R^d$ by
\begin{equation*}
	f^k(x) = A^k x + \sum_{i=1}^{k} A^{k-i} \eta (f^{i-1}(x)),
\end{equation*}%
which can easily be shown by induction on $k$.
This can be written in the center-stable and unstable components as
\begin{equation}
\label{TH_orbit_components}
\begin{split}
	(f^k)^+(x) &= B^k x^+ + \sum_{i=1}^{k} B^{k-i} \eta^+(f^{i-1}(x)), \\
	(f^k)^-(x) &= C^k x^- + \sum_{i=1}^{k} C^{k-i} \eta^-(f^{i-1}(x)).
\end{split}
\end{equation}%
In particular, we obtain
\begin{equation}
\label{TH_solve_for_xminus_1}
	x^- = C^{-k}(f^k)^-(x) - \sum_{i=1}^{k} C^{-i} \eta^-(f^{i-1}(x)).
\end{equation}%
Next, define $\mu = \min_{j \in \{s+1,\dots,d\}} |\lambda_j| = \norm{C^{-1}}^{-1} \in (1,\infty)$ and let $\weight = (\weight_k)_{k \geq 0} \subseteq (0,1]$ be given by $\weight_k = \mu^{-k/2}$.
Note that the space $\sequW = \{ (x_k)_{k \geq 0} \subseteq \R^d \colon \sup_{k \geq 0} \weight_k \norm{x_k} < \infty \}$ equipped with $\WNorm{(x_k)_{k \geq 0}} = \sup\nolimits_{k \geq 0} \weight_k \norm{x_k}$ is a Banach space since it is isomorphic to the Banach space $\ell^{\infty}$ of bounded sequences.
If $x = (x_k)_{k \geq 0} \in \sequW$, then
\begin{equation*}
	\norm{C^{-k}x_k^-} \leq \norm{C^{-k}} \norm{x_k^-} = \weight_k^2 \norm{x_k^-} \leq \weight_k \WNorm{x} \xrightarrow{k \rightarrow \infty} 0.
\end{equation*}%
Let us introduce the orbit map $\Orb$ to the space of sequences $\sequ = \{(x_k)_{k \geq 0} \subseteq \R^d\}$;
\begin{equation*}
	\Orb \colon \R^d \rightarrow \sequ, \quad x \mapsto \Orb x = (f^k(x) )_{k \geq 0}.
\end{equation*}%
If $\Orb x \in \sequW$, then $\norm{C^{-k} (f^k)^-(x)} \rightarrow 0$ as $k \rightarrow \infty$, so the partial sums in \eqref{TH_solve_for_xminus_1} converge in this case.
Thus, if $\Orb x \in \sequW$, then
\begin{equation*}
	x^- = - \sum_{i=1}^{\infty} C^{-i} \eta^-(f^{i-1}(x)).
\end{equation*}%
Plugging this into \eqref{TH_orbit_components} yields
\begin{equation}
\label{TH_solve_for_xminus_2}
	(f^k)^-(x) = - \sum_{i=k+1}^{\infty} C^{k-i} \eta^-(f^{i-1}(x))
\end{equation}
for all $x \in \R^d$ with $\Orb x \in \sequW$.
Let $\rho_r$ be the functions promised by Lemma \ref{bump_function} and let $r_{\epsilon}$ be the radii from Assumption \ref{center_manifold_asmp}.
Since $\rho_{r_{\epsilon}}(\R^d) \subseteq \B_{r_{\epsilon}}(0)$ and since $\eta$ is $\epsilon$-Lipschitz continuous on $\B_{r_{\epsilon}}(0)$ by assumption with $\eta(0) = 0$, we have, for all $x \in \sequW$ and $k \in \N$,
\begin{equation}
\label{TH_operator_welldef}
\begin{split}
	\sum_{i=k+1}^{\infty} \norm{C^{k-i} \eta^-(\rho_{r_{\epsilon}}(x_{i-1}))}
	&\leq \sum_{i=k+1}^{\infty} \norm{C^{k-i}} \norm{\eta^-(\rho_{r_{\epsilon}}(x_{i-1})) - \eta^-(\rho_{r_{\epsilon}}(0))} \\
	&\leq 6 \epsilon \sqrt{d} \sum_{i=k+1}^{\infty} \norm{C^{k-i}} \norm{x_{i-1}}
	= 6 \epsilon \sqrt{d} \sum_{i=k+1}^{\infty} \weight_{i-k} \weight_{i-1} \norm{x_{i-1}} \weight_{k-1}^{-1} \\
	&\leq 6 \epsilon \sqrt{d} \WNorm{x} \weight_{k-1}^{-1} \sum_{i=k+1}^{\infty} \weight_{i-k}
	= 6 \epsilon \sqrt{d} \WNorm{x} \weight_{k-1}^{-1} \frac{\weight_1}{1-\weight_1} < \infty
\end{split}
\end{equation}%
and for $k=0$
\begin{equation}
\label{TH_operator_welldef_2}
\begin{split}
	\sum_{i=1}^{\infty} \norm{C^{-i} \eta^-(\rho_{r_{\epsilon}}(x_{i-1}))}
	&\leq \norm{C^{-1} \eta^-(\rho_{r_{\epsilon}}(x_0))} + \norm{C^{-1}} \sum_{i=2}^{\infty} \norm{C^{1-i} \eta^-(\rho_{r_{\epsilon}}(x_{i-1}))} \\
	&\leq \norm{C^{-1}} 6 \epsilon \sqrt{d} \WNorm{x} + \norm{C^{-1}} 6 \epsilon \sqrt{d} \WNorm{x} \frac{\weight_1}{1-\weight_1}
	= 6 \epsilon \sqrt{d} \WNorm{x} \frac{\weight_1^2}{1-\weight_1}.
\end{split}
\end{equation}%
Hence, for all $\epsilon \in (0,1)$ and $y \in E^{cs}$, the map $\Op{y}{\epsilon} \colon \sequW \rightarrow \sequ$ given by
\begin{equation*}
	(\Op{y}{\epsilon}x)_k =
	\begin{pmatrix}
		B^k y + \sum_{i=1}^{k} B^{k-i} \eta^+(\rho_{r_{\epsilon}}(x_{i-1})) \\
		- \sum_{i=k+1}^{\infty} C^{k-i} \eta^-(\rho_{r_{\epsilon}}(x_{i-1}))
	\end{pmatrix} \in E^{cs} \oplus E^u
\end{equation*}%
for all $k \geq 0$ is well-defined.
We write $\U{\epsilon} \subseteq \R^d$ for the set $\U{\epsilon} = \{x \in \R^d \colon f^k(x) \in \B_{r_{\epsilon}/2}(0) \text{ for all } k \in \N_0\}$.
In \eqref{TH_orbit_components} and \eqref{TH_solve_for_xminus_2} above, we established that if $x \in \U{\epsilon}$ (which implies $\Orb x \in \sequW$), then $\Orb x$ is a fixed point of $\Op{x^+}{\epsilon}$.
Since $\norm{B^j} \leq 1$ for all $j \in \N_0$, we have, for all $x \in \sequW$ and $k \in \N_0$,
\begin{equation*}
\begin{split}
	&\weight_k \norm{ B^k y + \sum_{i=1}^{k} B^{k-i} \eta^+(\rho_{r_{\epsilon}}(x_{i-1})) }
	\\
	&\leq \weight_k \norm{y} + \weight_k \sum_{i=1}^k \norm{ \eta^+(\rho_{r_{\epsilon}}(x_{i-1}))}
	\leq \weight_k \norm{y} + 6 \epsilon \sqrt{d} \weight_k \sum_{i=1}^k \norm{ x_{i-1} }
	= \weight_k \norm{y} + 6 \epsilon \sqrt{d} \sum_{i=1}^k \weight_{k-i+1} \weight_{i-1} \norm{ x_{i-1} } \\
	&\leq \weight_k \norm{y} + 6 \epsilon \sqrt{d} \WNorm{x} \sum_{i=1}^k \weight_{k-i+1}
	= \weight_k \norm{y} + 6 \epsilon \sqrt{d} \WNorm{x} \weight_1 \frac{1-\weight_k}{1-\weight_1}
	\leq \norm{y} + 6 \epsilon \sqrt{d} \WNorm{x} \frac{\weight_1}{1-\weight_1}.
\end{split}
\end{equation*}%
Together with \eqref{TH_operator_welldef} and \eqref{TH_operator_welldef_2}, we obtain, for all $y \in E^{cs}$, $x \in \sequW$, and $k \in \N_0$,
\begin{equation*}
	\weight_k \norm{(\Op{y}{\epsilon}x)_k} \leq \norm{y} + 12 \epsilon \sqrt{d} \WNorm{x} \frac{\weight_1}{1-\weight_1},
\end{equation*}%
so $\Op{y}{\epsilon}(\sequW) \subseteq \sequW$.
By essentially the same calculations, we find, for all $x^1,x^2 \in \sequW$ and $k \in \N_0$,
\begin{equation*}
	\weight_k \norm{(\Op{y}{\epsilon}x^1)_k - (\Op{y}{\epsilon}x^2)_k} \leq 12 \epsilon \sqrt{d} \WNorm{x^1-x^2} \frac{\weight_1}{1-\weight_1}.
\end{equation*}%
In other words, the restriction $\Op{y}{\epsilon} \colon \sequW \rightarrow \sequW$ is $12\epsilon\sqrt{d} \weight_1 (1-\weight_1)^{-1}$ Lipschitz continuous with respect to $\WNorm{\cdot}$.
In particular, for all $y \in E^{cs}$ and $\epsilon \in (0,(1-\weight_1)(12\sqrt{d}\weight_1)^{-1})$, the restriction $\Op{y}{\epsilon} \colon \sequW \rightarrow \sequW$ is a contraction.
Now, let $\epsilon = (1-\weight_1)(24\sqrt{d}\weight_1)^{-1}$.
By the Banach Fixed Point Theorem, there is a unique fixed point map $\FPmap \colon E^{cs} \rightarrow \sequW$ specified by $\Op{y}{\epsilon}\FPmap(y) = \FPmap(y)$.
Note that, for any $y_1,y_2 \in E^{cs}$, $x \in \sequW$, and $k \in \N_0$,
\begin{equation*}
	(\Op{y_1}{\epsilon}x - \Op{y_2}{\epsilon}x)_k =
	\begin{pmatrix}
		B^k(y_1-y_2) \\ 0
	\end{pmatrix} \in E^{cs} \oplus E^u.
\end{equation*}%
Thus,
\begin{equation*}
\begin{split}
	\WNorm{\FPmap(y_1)-\FPmap(y_2)}
	&\leq \WNorm{\Op{y_1}{\epsilon}\FPmap(y_1) - \Op{y_2}{\epsilon}\FPmap(y_1)} + \WNorm{\Op{y_2}{\epsilon}\FPmap(y_1) - \Op{y_2}{\epsilon}\FPmap(y_2)}
	\\
	&\leq \norm{y_1-y_2} + \frac{1}{2} \WNorm{\FPmap(y_1)-\FPmap(y_2)}
\end{split}
\end{equation*}%
and, hence,
\begin{equation*}
	\WNorm{\FPmap(y_1)-\FPmap(y_2)} \leq 2 \norm{y_1-y_2}.
\end{equation*}%
So, $\FPmap \colon E^{cs} \rightarrow \sequW$ is Lipschitz continuous.
Denote by $\CSMmap \colon E^{cs} \rightarrow E^u$ the map $\CSMmap(y) = (\FPmap(y))_0^-$.
Then, for all $y_1,y_2 \in E^{cs}$,
\begin{equation*}
	\norm{\CSMmap(y_1)-\CSMmap(y_2)} \leq \WNorm{\FPmap(y_1)-\FPmap(y_2)} \leq 2 \norm{y_1-y_2},
\end{equation*}%
so $\CSMmap$ is also Lipschitz continuous.
We noted above that if $x \in \U{\epsilon}$, then $\Orb x$ is a fixed point of $\Op{x^+}{\epsilon}$.
Thus, if $x \in \U{\epsilon}$, then $\Orb x = \FPmap(x^+)$ and $x^- = \CSMmap(x^+)$.
In other words, we have shown that $\U{\epsilon} \subseteq \mathrm{Graph}(\CSMmap)$.
This proves Theorem \ref{center_manifold_thrm} in the diagonal case.


\subsection{Proof of the theorem in the general case}

In this section, we prove Theorem \ref{center_manifold_thrm} in the general case by reducing it to the special case from the previous section.
As before, denote $s = \dim(E^{cs}_z)$.
Since $f'(z)$ is diagonalizable, there is an invertible matrix $Q \in \R^{d \times d}$ such that $Q^{-1} f'(z) Q$ is a diagonal matrix, of which the first $s$ entries lie in $[-1,1]$ and the last $d-s$ entries lie in $\R \backslash [-1,1]$.
Set $\tilde{f}(x) = Q^{-1}f(z+Qx)-Q^{-1}z$.
Given $\epsilon \in (0,\infty)$, set $\delta(\epsilon) = \epsilon / (\norm{Q}\norm{Q^{-1}})$ and $\tilde{r}_{\epsilon} = r_{\delta(\epsilon)}/\norm{Q}$, where $r_{\delta}$ are the radii from Assumption \ref{center_manifold_asmp}.
Then, $\tilde{f}$ and $\tilde{r}_{\epsilon}$ satisfy Assumption \ref{center_manifold_asmp} at the point $0$.
Indeed, if $x,y \in \B_{\tilde{r}_{\epsilon}}(0)$, then $z+Qx,z+Qy \in \B_{r_{\delta(\epsilon)}}(z)$ and
\begin{equation*}
\begin{split}
	&\norm{ \tilde{f}(x) - \tilde{f}'(0) x - \big( \tilde{f}(y) - \tilde{f}'(0) y \big) } \\
	&= \norm{ Q^{-1} \big[ f(z+Qx) - z - f'(z)(z+Qx-z) - \big( f(z+Qy) - z - f'(z)(z+Qy-z) \big) \big] } \\
	&\leq \norm{ Q^{-1} } \delta(\epsilon) \norm{z+Qx-(z+Qy)}
	\leq \epsilon.
\end{split}
\end{equation*}%
By the theorem for the diagonal case, there exist an $\tilde{r} \in (0,\infty)$ and a Lipschitz continuous map $\tilde{\CSMmap} \colon \tilde{E}^{cs} \rightarrow \tilde{E}^u$ such that $\{x \in \R^d \colon \tilde{f}^k(x) \in \B_{\tilde{r}}(0) \text{ for all } k \in \N_0\} \subseteq \mathrm{Graph}(\tilde{\CSMmap})$.
Note that $E^{cs}_z = Q\tilde{E}^{cs}$ and $E^u_z = Q\tilde{E}^u$.
Now, set $r = \tilde{r}/\norm{Q^{-1}}$.
Observe that $\tilde{f}^k(x) = Q^{-1}f^k(z+Qx)-Q^{-1}z$ for all $x \in \R^d$ and $k \in \N_0$.
In particular, if $f^k(x) \in \B_{r}(z)$, then $\tilde{f}^k(Q^{-1}(x-z)) \in \B_{\tilde{r}}(0)$.
Thus, if $y \in \{x \in \R^d \colon f^k(x) \in \B_{r}(z) \text{ for all } k \in \N_0\}$, then $Q^{-1}(y-z) \in \mathrm{Graph}(\tilde{\CSMmap})$ and, hence, $y \in Q(\mathrm{Graph}(\tilde{\CSMmap}))+z$.
Define $\CSMmap \colon E^{cs}_z \rightarrow E^u_z$ by
\begin{equation*}
	\CSMmap(x) = Q\tilde{\CSMmap}(Q^{-1}(x-\Pi^+(z)))+\Pi^-(z),
\end{equation*}%
where $\Pi^+ \colon \R^d \rightarrow E^{cs}_z$ and $\Pi^- \colon \R^d \rightarrow E^u_z$ are the projections given by $\Pi^{\pm}(x) = Q \tilde{\Pi}^{\pm}(Q^{-1}x)$.
Then, $Q(\mathrm{Graph}(\tilde{\CSMmap}))+z = \mathrm{Graph}(\CSMmap)$, which finishes the proof of Theorem \ref{center_manifold_thrm}.


\newpage%
\phantomsection%
\addcontentsline{toc}{section}{\protect Acknowledgments}%
\vskip 2em \noindent{\large\bfseries Acknowledgments}
\vskip 2ex \noindent
This work has been partially funded by an Eric and Wendy Schmidt AI in Science Postdoctoral Fellowship.
This work has also been partially funded by the European Union (ERC, MONTECARLO, 101045811).
The views and the opinions expressed in this work are however those of the authors only and do not necessarily reflect those of the European Union or the European Research Council.
Neither the European Union nor the granting authority can be held responsible for them.
In addition, the second author gratefully acknowledges the Cluster of Excellence EXC 2044-390685587, Mathematics M{\"u}nster: Dynamics-Geometry-Structure funded by the Deutsche Forschungsgemeinschaft (DFG, German Research Foundation).

\phantomsection%
\addcontentsline{toc}{section}{\protect References}%
\bibliographystyle{acm}%
\bibliography{bibfile_FR_2024_sep_11}

\end{document}